\documentclass[10pt,twocolumn,letterpaper]{article}

\usepackage{cvpr}
\usepackage{times}
\usepackage{epsfig}
\usepackage{epstopdf} 

\usepackage{graphicx}
\usepackage{amsmath}
\usepackage{amssymb}
\usepackage{epstopdf} 

\usepackage{amsthm, color, listings, comment}
\usepackage{algorithm}
\usepackage{algorithmic}
\usepackage{appendix}
\usepackage{bm}
\usepackage{microtype}
\usepackage{subfigure}
\usepackage{authblk}

\newcommand{\norm}[1]{\left\|#1\right\|}

\newtheorem{lemma}{Lemma}
\newtheorem{theorem}[lemma]{Theorem}
\newtheorem{proposition}[lemma]{Proposition}
\newtheorem{corollary}[lemma]{Corollary}
\newtheorem{remark}[lemma]{Remark}
\newtheorem*{theorem*}{Theorem}   
\newtheorem*{proposition*}{Proposition}   
\newtheorem*{corollary*}{Corollary}   

\usepackage[pagebackref=true,breaklinks=true,letterpaper=true,colorlinks,bookmarks=false]{hyperref}

 \cvprfinalcopy 

\def\cvprPaperID{1428} 

\ifcvprfinal\fi
\begin{document}

\title{A Sufficient Condition for Convergences of Adam and RMSProp}
\author{ \vspace{-0.8cm}
Fangyu Zou{$^\dag$$^*$}, Li Shen$^\ddag$\thanks{The first two authors contribute equally.
$^\dag$This work was partially done when Fangyu Zou was a research intern at Tencent AI Lab, China.}~~, Zequn Jie$^\ddag$, Weizhong Zhang$^\ddag$, Wei Liu$^\ddag$

$^\ddag$Tencent AI Lab\qquad  $^\dag$Stony Brook University

{\tt\small fangyu.zou@stonybrook.edu,\  mathshenli@gmail.com,\  zequn.nus@gmail.com,
zhangweizhongzju@gmail.com,\  wl2223@columbia.edu
}
\vspace{-0.2cm}
}

\maketitle
\thispagestyle{empty}

\begin{abstract}
Adam and RMSProp are two of the most influential adaptive stochastic algorithms for training deep neural networks, which have been pointed out to be divergent even in the convex setting via a few simple counterexamples.
Many attempts, such as decreasing an adaptive learning rate, adopting a big batch size, incorporating a temporal decorrelation technique, seeking an analogous surrogate, \textit{etc.}, have been tried to promote Adam/RMSProp-type algorithms to converge.
In contrast with existing approaches, we introduce an alternative easy-to-check sufficient condition, which merely depends on the parameters of the base learning rate and combinations of historical second-order moments, to guarantee the global convergence of generic Adam/RMSProp for solving large-scale non-convex stochastic optimization.
Moreover, we show that the convergences of several variants of Adam, such as AdamNC, AdaEMA, \textit{etc.}, can be directly implied via the proposed sufficient condition in the non-convex setting.
In addition, we illustrate that Adam is essentially a specifically weighted AdaGrad with exponential moving average momentum, which provides a novel perspective for understanding Adam and RMSProp.
This observation coupled with this sufficient condition gives much deeper interpretations on their divergences.
At last, we validate the sufficient condition by applying Adam and RMSProp to tackle a certain counterexample and train deep neural networks.
Numerical results are exactly in accord with our theoretical analysis.
\vspace{-0.2cm}
\end{abstract}

\section{Introduction}\label{introduction}
Large-scale non-convex stochastic optimization \cite{bottou2018optimization}, covering a slew of applications in statistics and machine learning \cite{jain2017non,bottou2018optimization} such as learning a latent variable from massive data whose probability density distribution is unknown, takes the following generic formulation:
\begin{align}\label{minimization}
\min_{\bm{x} \in \mathbb{R}^{d}}\ f(\bm{x}) = \mathbb{E}_{\xi\sim \mathbb{P}}\,\big[\widetilde{f}(\bm{x},\xi)\big],
\end{align}
where $f(\bm{x})$ is a non-convex function and $\xi$ is a random variable satisfying an unknown distribution $\mathbb{P}$.

Due to the uncertainty of distribution $\mathbb{P}$, the batch gradient descent algorithm \cite{bertsekas1999nonlinear} is impractical to employ full gradient $\bm{\nabla}\!{f}(\bm{x})$ to solve problem \eqref{minimization}.
Alternatively, a compromised approach to handle this difficulty is to use an unbiased stochastic estimate of $\bm{\nabla}\!{f}(\bm{x})$, denoted as $g(\bm{x},\xi)$, which leads to the stochastic gradient descent (SGD) algorithm \cite{robbins1985stochastic}. Its coordinate-wise version is defined as:
\begin{align}\label{sgd}
\bm{x}_{t+1,k} = \bm{x}_{t,k} - \eta_{t,k}\bm{g}_{t,k}(\bm{x}_{t},\xi_{t}),
\end{align}
for $ k =1,2,\ldots,d$, where $\eta_{t,k} \ge 0$ is the learning rate of the $k$-th component of stochastic gradient  $\bm{g}(\bm{x}_{t},\xi_{t})$ at the $t$-th iteration. A sufficient condition \cite{robbins1985stochastic} to ensure the global convergence of vanilla SGD \eqref{sgd} is to require $\eta_{t}$ to meet
\begin{align}\label{surfficient-sgd}
\textstyle\sum\limits_{t=1}^{\infty}\|\eta_{t}\| = \infty\ {\rm\ and\ } \ \sum\limits_{t=1}^{\infty}\|\eta_{t}\|^2 < \infty.
\end{align}
Although the vanilla SGD algorithm with learning rate $\eta_{t}$ satisfying condition \eqref{surfficient-sgd} does converge, its empirical performance could be still stagnating,  since it is difficult to tune an effective learning rate $\eta_{t}$ via condition \eqref{surfficient-sgd}.

To further improve the empirical performance of SGD, a large variety of adaptive SGD algorithms, including AdaGrad \cite{duchi2011adaptive}, RMSProp \cite{hinton2012neural}, Adam \cite{kingma2014adam}, Nadam \cite{dozat2016incorporating}, \textit{etc.}, have been proposed to automatically tune the learning rate $\eta_{t}$ by using second-order moments of historical stochastic gradients.
Let $\bm{v}_{t,k}$ and $\bm{m}_{t,k}$ be the linear combinations of the historical second-order moments $(\bm{g}^2_{1,k},\bm{g}^2_{2,k},\cdots,\bm{g}^2_{t,k})$ and stochastic gradient estimates $(\bm{g}_{1,k},\bm{g}_{2,k},\cdots,\bm{g}_{t,k})$, respectively.
Then, the generic iteration scheme of these adaptive SGD algorithms \cite{Reddi2018on,chen2018convergence} is summarized as
\begin{equation}\label{generic-adaptive-sgd}
\!\! \bm{x}_{t+1,k}\!= \bm{x}_{t,k} - \eta_{t,k}\bm{m}_{t,k},\ {\rm\ with\ }\eta_{t,k}\!= {\alpha_{t}}/{\sqrt{\bm{v}_{t,k}}},\!\!
\end{equation}
for $k =1,2,\ldots,d$, where $\alpha_{t} >0 $ is called base learning rate and it is independent of stochastic gradient estimates $(\bm{g}_{1,k},\bm{g}_{2,k},\cdots,\bm{g}_{t,k})$ for all $t\ge 1$.
Although RMSProp, Adam, and Nadam work well for solving large-scale convex and non-convex optimization problems such as training deep neural networks, they have been pointed out to be divergent in some scenarios via convex counterexamples \cite{Reddi2018on}.
This finding thoroughly destroys the fluke of a direct use of these algorithms without any further assumptions or corrections.
Recently, developing sufficient conditions to guarantee global convergences of Adam and RMSProp-type algorithms has attracted much attention from both machine learning and optimization communities. The existing successful attempts can be divided into four categories:

\smallskip
\noindent
{\bf (C1) Decreasing a learning rate.}\ ~
Reddi \etal \cite{Reddi2018on} have declared that the core cause of divergences of Adam and RMSProp is largely controlled by the difference between the two adjacent learning rates, \textit{i.e.},
\begin{equation}\label{Gamma_t}
\Gamma_{t} = {1}/{\bm{\eta}_{t}}-{1}/{\bm{\eta}_{t-1}}= {\sqrt{\bm{v}_{t}}}/{\alpha_{t}}-{\sqrt{\bm{v}_{t-1}}}/{\alpha_{t-1}}.
\end{equation}
Once positive definiteness of $\Gamma_{t}$ is violated, Adam and RMSProp may suffer from divergence \cite{Reddi2018on}.
Based on this observation, two variants of Adam called AMSGrad and AdamNC have been proposed with convergence guarantees in both the convex \cite{Reddi2018on} and non-convex \cite{chen2018convergence} stochastic settings by requiring $\Gamma_{t}\succ 0$.
In addition, Padam \cite{zhou2018convergence} extended from AMSGrad has been proposed to contract the generalization gap in training deep neural networks, whose convergence has been ensured by requiring $\Gamma_{t}\succ 0$.
In the strongly convex stochastic setting, by using the long-term memory technique developed in \cite{Reddi2018on}, Huang \etal \cite{huang2018nostalgic} have proposed NosAdam by attaching more weights on historical second-order moments to ensure its convergence.
Prior to that, the convergence rate of RMSProp \cite{mukkamala2017variants} has already been established in the convex stochastic setting by employing similar parameters to those of AdamNC \cite{Reddi2018on}.

\smallskip
\noindent
{\bf(C2) Adopting a big batch size.}\ ~ Basu \etal  \cite{basu2018convergence}, for the first time, showed that deterministic Adam and RMSProp with original iteration schemes are actually convergent by using full-batch gradient.
On the other hand, both Adam and RMSProp can be reshaped as specific signSGD-type algorithms \cite{balles18aDissecting,bernstein2018signSGG} whose $\mathcal{O}(1/\sqrt{T})$ convergence rates have been provided in the non-convex stochastic setting by setting batch size as large as the number of maximum iterations \cite{bernstein2018signSGG}.
Recently, Zaheer \etal \cite{Zaheer2018Adaptive} have established $\mathcal{O}(1/\sqrt{T})$ convergence rate of original Adam directly in the non-convex stochastic setting by requiring the batch size to be the same order as the number of maximum iterations.
We comment that this type of requirement is impractical when Adam and RMSProp are applied to tackle large-scale problems like \eqref{minimization}, since these approaches cost a huge number of computations to estimate big-batch stochastic gradients in each iteration.

\smallskip
\noindent
{\bf(C3) Incorporating a temporal decorrelation.}\ ~ By exploring the structure of the convex counterexample in \cite{Reddi2018on}, Zhou \etal \cite{zhou2018adashift} have pointed out that the divergence of RMSProp is fundamentally caused by the unbalanced learning rate rather than the absence of $\Gamma_{t}\succ 0$.
Based on this viewpoint, Zhou \etal \cite{zhou2018adashift} have proposed AdaShift by incorporating a temporal decorrelation technique to eliminate the inappropriate correlation between $\bm{v}_{t,k}$ and the current second-order moment $\bm{g}_{t,k}^2$,
in which the adaptive learning rate $\eta_{t,k}$ is required to be independent of $\bm{g}^2_{t,k}$.
However, convergence of AdaShift in \cite{zhou2018adashift} was merely restricted to RMSProp for solving the convex counterexample in \cite{Reddi2018on}.

\smallskip
\noindent
{\bf (C4) Seeking an analogous surrogate.}\ ~Due to the divergences of Adam and RMSProp \cite{Reddi2018on}, Zou \etal \cite{zou2018convergence} recently proposed a class of new surrogates called AdaUSM to approximate Adam and RMSProp by integrating weighted AdaGrad with unified heavy ball and Nesterov accelerated gradient momentums.
Its $\mathcal{O}(\log{(T)}/\sqrt{T})$ convergence rate has also been provided in the non-convex stochastic setting by requiring a non-decreasing weighted sequence.
Besides, many other adaptive stochastic algorithms without combining momentums, such as AdaGrad \cite{ward2018adagrad,li2018convergence} and stagewise AdaGrad \cite{chen2018universal}, have been guaranteed to be convergent and work well in the non-convex stochastic setting.

In contrast with the above four types of modifications and restrictions, we introduce an alternative easy-to-check sufficient condition (abbreviated as {\bf (SC)}) to guarantee the global convergences of original Adam and RMSProp.
The proposed {\bf (SC)} merely depends on the parameters in estimating $\bm{v}_{t,k}$ and base learning rate $\alpha_{t}$. {\bf (SC)} neither requires the positive definiteness of $\Gamma_{t}$ like {\bf(C1)} nor needs the batch size as large as the same order as the number of maximum iterations like {\bf(C2)} in both the convex and non-convex stochastic settings.
Thus, it is easier to verify and more practical compared with {\bf (C1)}-{\bf(C3)}.
On the other hand, {\bf (SC)} is partially overlapped with {\bf(C1)} since the proposed {\bf (SC)} can cover AdamNC \cite{Reddi2018on}, AdaGrad with exponential moving average  (AdaEMA) momentum \cite{chen2018convergence}, and RMSProp \cite{mukkamala2017variants} as instances whose convergences are all originally motivated by requiring the positive definiteness of $\Gamma_{t}$.
While, based on {\bf (SC)}, we can directly derive their global convergences in the non-convex stochastic setting as byproducts without checking the positive definiteness of $\Gamma_{t}$ step by step.
Besides, {\bf (SC)} can serve as an alternative explanation on divergences of original Adam and RMSProp, which are possibly due to incorrect parameter settings for accumulating the historical second-order moments rather than the unbalanced learning rate caused by the inappropriate correlation between $\bm{v}_{t,k}$ and $\bm{g}^2_{t,k}$ like {\bf(C3)}.
In addition, AdamNC and AdaEMA are convergent under {\bf(SC)}, but violate {\bf(C3)} in each iteration.

Moreover, by carefully reshaping the iteration scheme of Adam, we obtain a specific weighted AdaGrad with exponential moving average momentum, which extends the weighted AdaGrad with heavy ball momentum and Nesterov accelerated gradient momentum \cite{zou2018convergence} in two aspects: the new momentum mechanism and the new base learning rate setting provide a new perspective for understanding Adam and RMSProp.
At last, we experimentally verify {\bf (SC)} by applying Adam and RMSProp with different parameter settings to solve the counterexample \cite{Reddi2018on} and train deep neural networks including LeNet \cite{lecun1998gradient} and ResNet \cite{he2016deep}.
In summary, the contributions of this work are five-fold:
\begin{itemize}
\item[(1)] We introduce an easy-to-check sufficient condition to ensure the global convergences of original Adam and RMSProp in the non-convex stochastic setting. Moreover, this sufficient condition is distinctive from the existing conditions {\bf(C1)}-{\bf(C4)} and is easier to verify.
\item[(2)] We reshape Adam as weighted AdaGrad with exponential moving average momentum, which provides a new perspective for understanding Adam and RMSProp and also complements AdaUSM in \cite{zou2018convergence}.
\item[(3)] We provide a new explanation on the divergences of original Adam and RMSProp, which are possibly due to an incorrect parameter setting of the combinations of historical second-order moments based on {\bf (SC)}.
\item[(4)] We find that the sufficient condition extends the restrictions of RMSProp \cite{mukkamala2017variants} and covers many convergent variants of Adam, \eg, AdamNC, AdaGrad with momentum, \textit{etc.} Thus, their convergences in the non-convex stochastic setting naturally hold.
\item[(5)] We conduct experiments to validate the sufficient condition for the convergences of Adam/RMSProp. The experimental results match our theoretical results.
\end{itemize}

\section{Generic Adam}

For readers' convenience, we first clarify a few necessary notations used in the forthcoming Generic Adam.
We denote $\bm{x}_{t,k}$ as the $k$-th component of $\bm{x}_{t}\in\mathbb{R}^{d}$, and $\bm{g}_{t,k}$ as the $k$-th component of the stochastic gradient at the $t$-th iteration,
and call $\alpha_{t} > 0$ base learning rate and $\beta_{t}$ momentum parameter, respectively. Let $\epsilon>0$ be a sufficiently small constant. Denote $\bm{0}=(0,\cdots,0)^{\top} \in \mathbb{R}^{d}$, and $\bm{\epsilon}=(\epsilon,\cdots,\epsilon)^{\top} \in \mathbb{R}^{d}$.
All operations, such as multiplying, dividing, and taking square root, are executed in the coordinate-wise sense.

\begin{algorithm}[H]
\caption{\ Generic Adam}
\label{Adam}
\begin{algorithmic}[1]
   \STATE {\bf Parameters:} Choose $\{\alpha_t\}$, $\{\beta_t\}$, and $\{\theta_t\}$. Choose $\bm{x}_1 \in \mathbb{R}^d$ and set initial values $\bm{m}_0\!=\!\bm{0}$ and $\bm{v}_0=\bm{\epsilon}$.
   \FOR{$t= 1,2,\ldots,T$}
    \STATE Sample a stochastic gradient $\bm{g}_t$;
        \FOR {$k=1,2,\ldots,d$}
        \STATE $\bm{v}_{t,k} = \theta_t \bm{v}_{t-1,k} + (1 - \theta_t) \bm{g}_{t,k}^2$;
        \STATE $\bm{m}_{t,k} = \beta_t \bm{m}_{t-1,k} + (1 - \beta_t) \bm{g}_{t,k}$;
        \STATE $\bm{x}_{t+1,k} = \bm{x}_{t,k} - {\alpha_t \bm{m}_{t,k}}/\sqrt{\bm{v}_{t,k}}$;
        \ENDFOR
   \ENDFOR
 \end{algorithmic}
 \end{algorithm}
\vspace{-0.2cm}
Generic Adam covers RMSProp by setting $\beta_{t} =0$. Moreover, it covers Adam with a bias correction \cite{kingma2014adam} as follows:

\begin{remark}
The original Adam with the bias correction \cite{kingma2014adam} takes constant parameters $\beta_{t}=\beta$ and $\theta_{t}=\theta$. The iteration scheme is written as
$\bm{x}_{t+1} \!=\! \bm{x}_{t} \!-\! \widehat{\alpha}_{t} \frac{\widehat{\bm{m}}_{t}}{\sqrt{\widehat{\bm{v}}_{t}}}$, with $\widehat{\bm{m}}_{t} \!=\! \frac{\bm{m}_{t}}{1-\beta^{t}}$ and $\widehat{\bm{v}}_{t} \!=\! \frac{\bm{v}_{t}}{1-\theta^{t}}$. Let $\alpha_t = \widehat{\alpha}_t\frac{\sqrt{1-\theta^t}}{1-\beta^t}$. Then, the above can be rewritten as $\bm{x}_{t+1} = \bm{x}_{t} - {\alpha_t \bm{m}_{t}}/\sqrt{\bm{v}_{t}}$. Thus, it is equivalent to taking constant $\beta_t$, constant $\theta_t$, and new base learning rate $\alpha_t$ in Generic Adam.
\end{remark}

\subsection{Weighted AdaGrad Perspective}

Now we show that Generic Adam can be reformulated as a new type of weighted AdaGrad algorithms with exponential moving average momentum (Weighted AdaEMA).
\begin{algorithm}[H]
\caption{\ Weighted AdaEMA}
\label{Weithed AdamEMA}
\begin{algorithmic}[1]
   \STATE {\bf Parameters:} Choose parameters $\{\alpha_t\}$, momentum factors $\{\beta_t\}$, and weights $\{w_t\}$.
   Set $W_0 = 1$, $\bm{m}_0\!=\!\bm{0}$, $\bm{V}_0=\bm{\epsilon}$, and $\bm{x}_1 \in \mathbb{R}^n$.
   \FOR{$t= 1,2,\ldots,T$}
    \STATE Sample a stochastic gradient $\bm{g}_t$;
        \STATE $W_t = W_{t-1} + w_t$; 
        \FOR {$k=1,2,\ldots,d$}
        \STATE $\bm{V}_{t,k} = \bm{V}_{t-1,k} + w_t\bm{g}_{t,k}^2$; 
        \STATE $\bm{m}_{t,k} = \beta_t \bm{m}_{t-1,k} + (1 - \beta_t) \bm{g}_{t,k}$;
        \STATE $\bm{x}_{t+1,k} = \bm{x}_{t,k} - \alpha_t \bm{m}_{t,k}/\sqrt{\bm{V}_{t,k}/{W_t}}$;
        \ENDFOR
   \ENDFOR
 \end{algorithmic}
 \end{algorithm}
\begin{remark}
The Weighted AdaEMA is a natural generalization of the AdaGrad algorithm. The classical AdaGrad is to take the weights $w_t = 1$, the momentum factors $\beta_t = 0$, and the parameters $\alpha_t = \eta/\sqrt{t+1}$ for constant $\eta$.
\end{remark}

The following proposition states the equivalence between Generic Adam and  Weighted AdaEMA.
\begin{proposition}\label{equivalence-theorem}
Algorithm \ref{Adam} and Algorithm \ref{Weithed AdamEMA} are equivalent.
\end{proposition}

\noindent
{\bf The divergence issue of Adam/RMSProp.}\ ~When $\theta_{t}$ is taken constant, \textit{i.e.}, $\theta_{t} = \theta$, Reddi \etal \cite{Reddi2018on} have  pointed out that Adam and RMSProp ($\beta_{t} =0$) can be divergent even in the convex setting. They conjectured that the divergence is possibly due to the uncertainty of positive definiteness of $\Gamma_{t}$ in Eq.~\eqref{Gamma_t}. This idea has motivated many new convergent variants of Adam by forcing $\Gamma_{t} \succ 0$. Recently, Zhou \etal \cite{zhou2018adashift} further argued that the nature of divergences of Adam and RMSProp is possibly due to the unbalanced learning rate $\eta_{t}$ caused by the inappropriate correlation between $v_{t,k}$ and $g^2_{t,k}$ by studying the counterexample in \cite{Reddi2018on}. However, this explanation can be  violated by many existing convergent Adam-type algorithms such as AdamNC, NosAdam \cite{huang2018nostalgic}, \textit{etc.} So far, there is no satisfactory explanation for the core reason of the divergence issue. We will provide more insights in Section \ref{insights-for-divergence} based on our theoretical analysis.

\section{Main Results}

In this section, we characterize the upper-bound of gradient residual of problem \eqref{minimization} as a function of parameters $(\theta_t,\alpha_t)$. Then the convergence rate of Generic Adam is derived directly by specifying appropriate parameters $(\theta_{t},\alpha_{t})$. Below, we state the necessary assumptions that are commonly used for analyzing the convergence of a stochastic algorithm for non-convex problems:
\begin{description}
\item[(A1)] The minimum value of problem \eqref{minimization} is lower-bounded, \textit{i.e.}, $f^* = \min_{\bm{x} \in \mathbb{R}^{d}}\ f(\bm{x}) > -\infty$;
\item[(A2)] The gradient of $f$ is $L$-Lipschitz continuous,
\textit{i.e.}, $\|\bm{\nabla}\!f(\bm{x})-\bm{\nabla}\!f(\bm{y})\| \le L\|\bm{x}-\bm{y}\|,\ \forall \bm{x},\bm{y}\in \mathbb{R}^{d}$;
\item[(A3)] The stochastic gradient $\bm{g}_{t}$ is an unbiased estimate, \textit{i.e.}, $\mathbb{E}\,[\bm{g}_{t}] = \bm{\nabla}\!{f}_{t}(\bm{x}_{t})$;
\item[(A4)] The second-order moment of stochastic gradient $\bm{g}_{t}$ is uniformly upper-bounded, \textit{i.e.}, $\mathbb{E}\,\|\bm{g}_{t}\|^2 \le G$.
\end{description}

To establish the upper-bound, we also suppose that the parameters $\{\beta_t\}$, $\{\theta_t\}$, and $\{\alpha_t\}$ satisfy the restrictions:
\begin{description}
\item[(R1)] The parameters $\{\beta_t\}$ satisfy $0\leq \beta_t \leq \beta < 1$ for all $t$ for some constant $\beta$;
\item[(R2)] The parameters $\{\theta_t\}$ satisfy $0 < \theta_t < 1$ and $\theta_t$ is non-decreasing in $t$ with $\theta := \lim_{t\to\infty}\theta_t > \beta^2$;
\item[(R3)] The parameters $\{\alpha_t\}$ satisfy that $\chi_t := \frac{\alpha_t}{\sqrt{1-\theta_t}}$ is ``almost" non-increasing in $t$, by which we mean that there exist a non-increasing sequence $\{a_t\}$ and a positive constant $C_0$ such that $a_t \leq \chi_t \leq C_0 a_t$.
\end{description}

The restriction (R3) indeed says that $\chi_t$ is the product between some non-increasing sequence $\{a_t\}$ and some bounded sequence. This is a slight generalization of $\chi_t$ itself being non-decreasing. If $\chi_t$ itself is non-increasing, we can then take $a_t \!=\! \chi_t$ and $C_0 \!=\! 1$. For most of the well-known Adam-type methods, $\chi_t$ is indeed non-decreasing, for instance, for AdaGrad with EMA momentum we have $\alpha_t = \eta/\sqrt{t}$ and $\theta_t \!=\! 1 - 1/t$, so $\chi_t = \eta$ is constant; for Adam with constant $\theta_t \!=\! \theta$ and non-increasing $\alpha_t$ (say $\alpha_t \!=\! \eta/\sqrt{t}$ or $\alpha_t \!=\! \eta$), $\chi_t \!=\! \alpha_t/\sqrt{1-\theta}$ is non-increasing.  The motivation, instead of $\chi_t$ being decreasing, is that it allows us to deal with the bias correction steps in Adam \cite{kingma2014adam}.

We fix a positive constant $\theta' >0$\footnote{In the special case that $\theta_t = \theta$ is constant, we can directly set $\theta' = \theta$.} such that $\beta^2 < \theta' < \theta$. Let $\gamma := {\beta^2}/{\theta'} < 1$ and
\begin{equation}\label{Constant-C1}
C_1 := \textstyle\prod_{j=1}^N \big(\frac{\theta_j}{\theta'}\big),
\end{equation}
where $N$ is the maximum of the indices $j$ with $\theta_j < \theta'$. The finiteness of $N$ is guaranteed by the fact that $\lim_{t\to\infty} \theta_t = \theta > \theta'$. When there are no such indices, \textit{i.e.}, $\theta_1 \geq \theta'$, we take $C_1 = 1$ by convention. In general, $C_1 \leq 1$.
Our main results on estimating gradient residual state as follows:
\begin{theorem}\label{convergence_in_expectation}
Let $\{\bm{x}_t\}$ be a sequence generated by Generic Adam for initial values $\bm{x}_1$, $\bm{m}_0 =\bm{0}$, and $\bm{v}_0 =\bm{\epsilon}$. Assume that $f$ and stochastic gradients $\bm{g}_t$ satisfy assumptions (A1)-(A4). Let $\tau$ be randomly chosen from $\{1, 2, \ldots, T\}$ with equal probabilities $p_\tau = 1/T$. We have
\begin{equation*}
\left(\mathbb{E}\left[\norm{\bm{\nabla} f(\bm{x}_\tau)}^{4/3} \right]\right)^{3/2}
\leq \frac{C + C'\sum_{t=1}^T\alpha_t\sqrt{1-\theta_t}}{T\alpha_T},
\end{equation*}
where $C'\!=\!{2C_0^2C_3d\sqrt{G^2\!+\!\epsilon d}}{\big/}{[(1\!-\!\beta)\theta_1]}$ and 
\begin{equation*}
\begin{split}
C=\frac{2C_0\sqrt{G^2\!+\!\epsilon d}}{1-\beta}\big[(C_4\!+\!C_3C_0d\chi_1\log\big(1\!+\! \frac{G^2}{\epsilon d}\big)\big],
\end{split}
\end{equation*}
where $C_4$ and $C_3$ are defined as $C_4 = f(x_1) - f^*$ and
$C_3 = \frac{C_0}{\sqrt{C_1}(1-\sqrt{\gamma})}\big[\frac{C_0^2\chi_1L}{C_1(1-\sqrt{\gamma})^2} + 2\big(\frac{\beta/(1-\beta)}{\sqrt{C_1(1-\gamma)\theta_1}}+1\big)^2G\big]$.
\end{theorem}

\begin{theorem}\label{thm1-001}
Suppose the same setting and hypothesis as Theorem \ref{convergence_in_expectation}. Let $\tau$ be randomly chosen from $\{1, 2, \ldots, T\}$ with equal probabilities $p_\tau = 1/T$. Then for any $\delta > 0$, the following bound holds with probability at least $1-\delta^{2/3}$:
\begin{equation*}
\begin{aligned}
\norm{\bm{\nabla}\!f(\bm{x}_\tau)}^2 \leq
\frac{C + C'\sum_{t=1}^T\alpha_t \sqrt{1-\theta_t}}{\delta T \alpha_T}
:= Bound(T),
\end{aligned}
\end{equation*}
where $C$ and $C'$ are defined as those in Theorem \ref{convergence_in_expectation}.
\end{theorem}

\begin{remark}
{\bf(i)}\ The constants $C$ and $C'$ depend on apriori known constants $C_0, C_1, \beta, \theta', G, L, \epsilon, d, f^*$ and $\theta_1, \alpha_1, \bm{x}_1$. \\
{\bf(ii)}\  Convergence in expectation in Theorem \ref{convergence_in_expectation} is slightly stronger than convergence in probability in Theorem \ref{thm1-001}.
Convergence in expectation is on the term
$(\mathbb{E}[\norm{\nabla\! f(\bm{x}_\tau)}^\frac{3}{2}])^\frac{4}{3}$, which
is slightly weaker than $\mathbb{E}[\norm{\nabla\! f(\bm{x}_\tau)}^2]$. The latter is adopted for most SGD variants with global learning rates, namely, the learning rate for each coordinate is the same. This is due to that $\frac{1}{\sum_{t=1}^T\!\alpha_t}\mathbb{E}\sum_{t=1}^T\!\alpha_t \norm{\nabla\! f(\bm{x}_t)}^2$ is exactly $\mathbb{E}[\norm{\nabla\! f(\bm{x}_\tau)}^2]$ if $\tau$ is randomly selected via distribution $\mathbb{P}(\tau\!=\!k)\!=\!\frac{\alpha_k}{\sum_{t=1}^T\!\alpha_t}$.
This does not apply to coordinate-wise adaptive methods because the learning rate for each coordinate is different, and hence unable to randomly select an index according to some distribution uniform for each coordinate.
On the other hand, the proofs of AMSGrad and AdaEMA \cite{chen2018convergence} are able to achieve the bound for $\mathbb{E}[\norm{\nabla\! f(\bm{x}_\tau)}^2]$.
This is due to the strong assumption $\norm{g_t}\!\leq\!G$ which results in a uniform lower bound for each coordinate of the adaptive learning rate $\eta_{t,k}\!\geq\!\alpha_t/G$. Thus, the proof of AMSGrad \cite{chen2018convergence} can be dealt with in a way similar to the case of global learning rate.
In our paper we use a coordinate-wise adaptive learning rate and assume a weaker assumption $\mathbb{E}[\norm{g_t}^2]\!\leq\!G$ instead of $\norm{g_t}^2\!\leq\!G$.
To separate the term $\norm{\nabla\! f(\bm{x}_t)}$ from $\norm{\nabla\! f(\bm{x}_t)}^2_{\bm{\hat{\eta}}_t}$, we can only apply the H\"{o}lder theorem to obtain a bound for $(\mathbb{E}[\norm{\nabla\! f(\bm{x}_\tau)}^\frac{3}{2}])^\frac{4}{3}$.
\end{remark}

\begin{corollary}\label{constant-theta}
Take $\alpha_t = \eta/t^s$ with $0\leq s < 1$. Suppose $\lim_{t\to\infty} \theta_t = \theta < 1$. Then the $Bound(T)$ in Theorem \ref{thm1-001} is bounded from below by constants
\begin{equation}
Bound(T) \geq \frac{C'\sqrt{1-\theta}}{\delta}.
\end{equation}
In particular, when $\theta_t = \theta < 1$, we have the following more subtle estimate on lower and upper-bounds for $Bound(T)$
\begin{equation*}
\frac{C}{\delta \eta T^{1-s}} +\frac{C'\sqrt{1-\theta}}{\delta}\leq Bound(T) \!\leq\! \frac{C}{\delta \eta T^{1-s}} \!+\! \frac{C'\sqrt{1\!-\!\theta}}{\delta(1-s)}.
\end{equation*}
\end{corollary}

\begin{remark}
{\bf(i)}\ Corollary \ref{constant-theta} shows that if $\lim_{t\to\infty}\theta_t = \theta < 1$, the bound in Theorem \ref{thm1-001} is only $\mathcal{O}(1)$, hence not guaranteeing convergence.
This result is not surprising as Adam with constant $\theta_t$ has already shown to be divergent \cite{Reddi2018on}. Hence, $\mathcal{O}(1)$ is its best convergence rate we can expect.
We will discuss this case in more details in Section \ref{insights-for-divergence}.\\
{\bf(ii)}\ Corollary \ref{constant-theta} also indicates that in order to guarantee convergence, the parameter has to satisfy $\lim_{t\to\infty} \theta_t = 1$.
Although we do not assume this in our restrictions (R1)-(R3), it turns out to be the consequence from our analysis. Note that if $\beta < 1$ in (R1) and $\lim_{t\to\infty}\theta_t =1$, then the restriction $\lim_{t\to\infty}\theta_t > \beta^2$ is automatically satisfied in (R2).
\end{remark}

We are now ready to give the Sufficient Condition (\textbf{SC}) for  convergence of Generic Adam based on Theorem \ref{thm1-001}.
\begin{corollary}[Sufficient Condition(\textbf{SC})]\label{convergence}
Generic Adam is convergent if the parameters $\{\alpha_t\}$, $\{\beta_t\}$, and $\{\theta_t\}$ satisfy
\begin{itemize}\setlength{\itemsep}{0pt}
\item[1.] $\beta_t \leq \beta < 1$;
\item[2.] $0 < \theta_t < 1$ and $\theta_t$ is non-decreasing in $t$;
\item[3.] $\chi_t := \alpha_t/\sqrt{1-\theta_t}$ is ``almost" non-increasing;
\item[4.] $\big({\sum_{t=1}^T\alpha_t\sqrt{1-\theta_t}}\big){\big /}\big({T\alpha_T}\big) = o(1)$.
\end{itemize}
\end{corollary}

\medskip
\subsection{Convergence Rate of Generic Adam}

We now provide the convergence rate of Generic Adam with a specific class of parameters $\{(\theta_{t},\alpha_{t})\}$, \textit{i.e.},
\begin{align}\label{convergence-rate-parameter}
\alpha_t = \eta/t^s \text{~~and~~}  \theta_t = \left\{
\begin{aligned}
& 1 - \alpha/K^r, \quad & t < K,\\
& 1 - \alpha/t^r, \quad & t \geq K,
\end{aligned}
\right.
\end{align}
for positive constants $\alpha,\eta,K$, where $K$ is taken such that $\alpha/K^r < 1$. Note that $\alpha$ can be taken bigger than 1. When $\alpha < 1$, we can take $K = 1$ and then $\theta_t = 1 - \alpha/t^r, t\geq 1$. To guarantee (R3), we require $r \leq 2s$. For such a family of parameters we have the following corollary.

\begin{corollary}\label{poly-setting}
Generic Adam with the above family of parameters converges as long as $0 < r \leq 2s < 2$, and its non-asymptotic convergence rate is given by
\begin{equation*}
\norm{\bm{\nabla} f(\bm{x}_\tau)}^2 \leq \left\{\begin{aligned}
& \mathcal{O}(T^{-r/2}), \quad &   r/2 + s < 1 \\
& \mathcal{O}(\log(T)/T^{1-s}), \quad &  r/2 + s = 1 \\
& \mathcal{O}(1/T^{1 - s}), \quad &  r/2 + s > 1
\end{aligned}\right..
\end{equation*}
\end{corollary}

\begin{remark}
Corollary \ref{poly-setting} recovers and extends the results of some well-known algorithms below:
\begin{itemize}
\item{\textbf{AdaGrad with exponential moving average (EMA)}.}
When $\theta_t = 1 - 1/t$, $\alpha_t = \eta/\sqrt{t}$, and $\beta_t = \beta < 1$, Generic Adam is exactly AdaGrad with EMA momentum (AdaEMA) \cite{chen2018convergence}.
In particular, if $\beta = 0$, this is the vanilla coordinate-wise AdaGrad.
It corresponds to taking $r = 1$ and $s = 1/2$ in Corollary \ref{poly-setting}. Hence, AdaEMA has convergence rate $\log(T)/\sqrt{T}$.

\item{\textbf{AdamNC}.}
Taking $\theta_t = 1-1/t$, $\alpha_t = \eta/\sqrt{t}$, and $\beta_t = \beta\lambda^t$ in Generic Adam, where $\lambda < 1$ is the decay factor for the momentum factors $\beta_t$, we recover AdamNC \cite{Reddi2018on}.
Its $\mathcal{O}(\log{(T)}/\sqrt{T})$ convergence rate can be directly derived via Corollary \ref{poly-setting}.

\item{\textbf{RMSProp}.}
Mukkamala and Hein \cite{mukkamala2017variants} have reached the same $\mathcal{O}(\log{(T)}/\sqrt{T})$ convergence rate for RMSprop with $\theta_t = 1 - \alpha/t$, when $0< \alpha \leq 1$ and $\alpha_t = \eta/\sqrt{t}$ under the convex assumption. Since RMSProp is essentially Generic Adam with all momentum factors $\beta_t = 0$, we recover Mukkamala and Hein's results by taking $r = 1$ and $s = 1/2$ in Corollary \ref{poly-setting}. Moreover, our result generalizes to the non-convex stochastic setting, and it holds for all $\alpha\!>\!0$ rather than only $0\! <\! \alpha\! \leq\! 1$.
\end{itemize}
\end{remark}

As Weighted AdaEMA is equivalent to Generic Adam, we present its convergence rate with specific polynomial growth weights in the following corollary.
\begin{corollary}\label{polyweights}
Suppose in Weighted AdaEMA the weights $w_t = t^r$ for $r\!\geq\! 0$, and $\alpha_t \!=\!\eta/\sqrt{t}$. Then Weighted AdaEMA has the $\mathcal{O}(\log(T)/\sqrt{T})$ non-asymptotic convergence rate.
\end{corollary}

\begin{remark}
Zou \etal \cite{zou2018convergence} proposed weighted AdaGrad with a unified momentum form which incorporates Heavy Ball (HB) momentum and Nesterov Accelerated Gradients (NAG) momentum. The same convergence rate was established for weights with polynomial growth. Our result complements \cite{zou2018convergence} by showing that the same convergence rate also holds for exponential moving average momentum.
\end{remark}
\begin{remark}
{\bf(i)}\ Huang \etal \cite{huang2018nostalgic} proposed Nostalgic Adam (NosAdam) which corresponds to taking the learning rate $\alpha_t = \eta/\sqrt{t}$ and $\theta_t = B_{t-1}/B_{t}$ with $B_t = \sum_{i=1}^t b_i$ for $b_i > 0,\ i \geq 0$, and $B_0 > 0$\footnote{We directly use $B_{t}$ and $b_{i}$ along with the notations of NosAdam \cite{huang2018nostalgic}.} in Generic Adam. The idea of NosAdam is to guarantee $\Gamma_{t}\succ 0$ by laying more weights on the historical second-order moments. A special case of NosAdam is NosAdam-HH which takes $B_t = \sum_{i=1}^t i^{-r}$ for $r \geq 0$  as the hyper-harmonic series. Its $\mathcal{O}(1/\sqrt{T})$ convergence rate is established in the strongly convex stochastic setting. NosAdam-HH can be viewed as the Weighted AdaEMA taking $\alpha_t = \eta/\sqrt{t}$ and $w_t = t^{-r}$ for $r \geq 0$.

\noindent
{\bf(ii)}\
Corollary \ref{polyweights} differs from the motivation of NosAdam as the weights we consider are $w_t = t^r$ for $r \geq 0$. Note that in both cases when $r = 0$, this is the AdaGrad algorithm, which corresponds to assigning equal weights to the past squares of gradients. Hence, we are actually in the opposite direction of NosAdam. 
We are more interested in the case of assigning more weights to the recent stochastic gradients. This can actually be viewed as a situation between AdaGrad and the original Adam with constant $\theta_t$'s.
\end{remark}

\noindent
{\bf Comparison between (SC) and (C1).}\ ~Most of the convergent modifications of original Adam, such as AMSGrad, AdamNC, and NosAdam, all require $\Gamma_t\succ 0$ in Eq.~\eqref{Gamma_t}, which is equivalent to decreasing the adaptive learning rate $\eta_t$ step by step. Since the term $\Gamma_t$ (or adaptive learning rate $\eta_t$) involves the past stochastic gradients (hence not deterministic), the modification to guarantee $\Gamma_t\succ 0$ either needs to change the iteration scheme of Adam (like AMSGrad) or needs to impose some strong restrictions on the base learning rate $\alpha_t$ and $\theta_t$ (like AdamNC). Our sufficient condition  provides an easy-to-check criterion for the convergence of Generic Adam in Corollary \ref{convergence}. It is not necessary to require $\Gamma_t\succ 0$. Moreover, we use exactly the same iteration scheme as  original Adam without any modifications. Our work shows that the positive definiteness of $\Gamma_t$ may not be the essential issue for divergence of original Adam. It is probably due to that the parameters are not set correctly.

\section{Constant $\theta_t$ case: insights for divergence}\label{insights-for-divergence}

The currently most popular RMSProp and Adam's parameter setting takes constant  $\theta_t$, \textit{i.e.}, $\theta_t = \theta < 1$. The motivation behind is to use the exponential moving average of squares of past stochastic gradients. In practice, parameter $\theta$ is recommended to be set very close to 1. For instance, a commonly adopted $\theta$ is taken as 0.999.

Although great performance in practice has been observed, such a constant parameter setting has the serious flaw that there is no convergence guarantee even for convex optimization, as proved by the counterexamples in \cite{Reddi2018on}.
Ever since much work has been done to analyze the divergence issue of Adam and to propose modifications with convergence guarantees, as summarized in the introduction section.
However, there is still not a satisfactory explanation that touches the fundamental reason of the divergence. Below, we try to provide more insights for the divergence issue of Adam/RMSProp with constant parameter $\theta_t$, based on our analysis of the sufficient condition for convergence.

\begin{figure*}[htpb]
\centering
\subfigure[]{\includegraphics[width=0.32\linewidth]{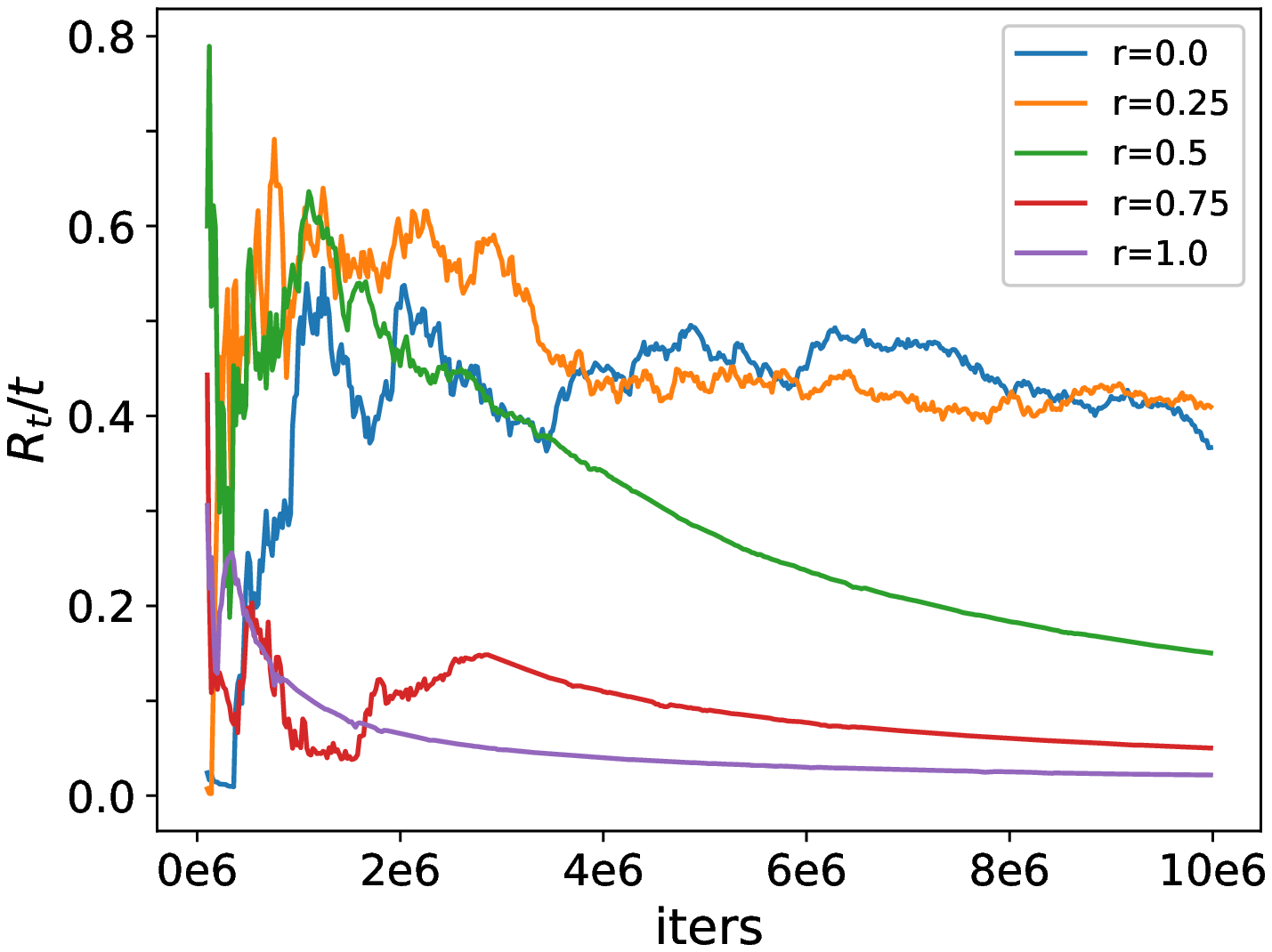}}\label{fig:synthesis_a}
\subfigure[]{\includegraphics[width=0.32\linewidth]{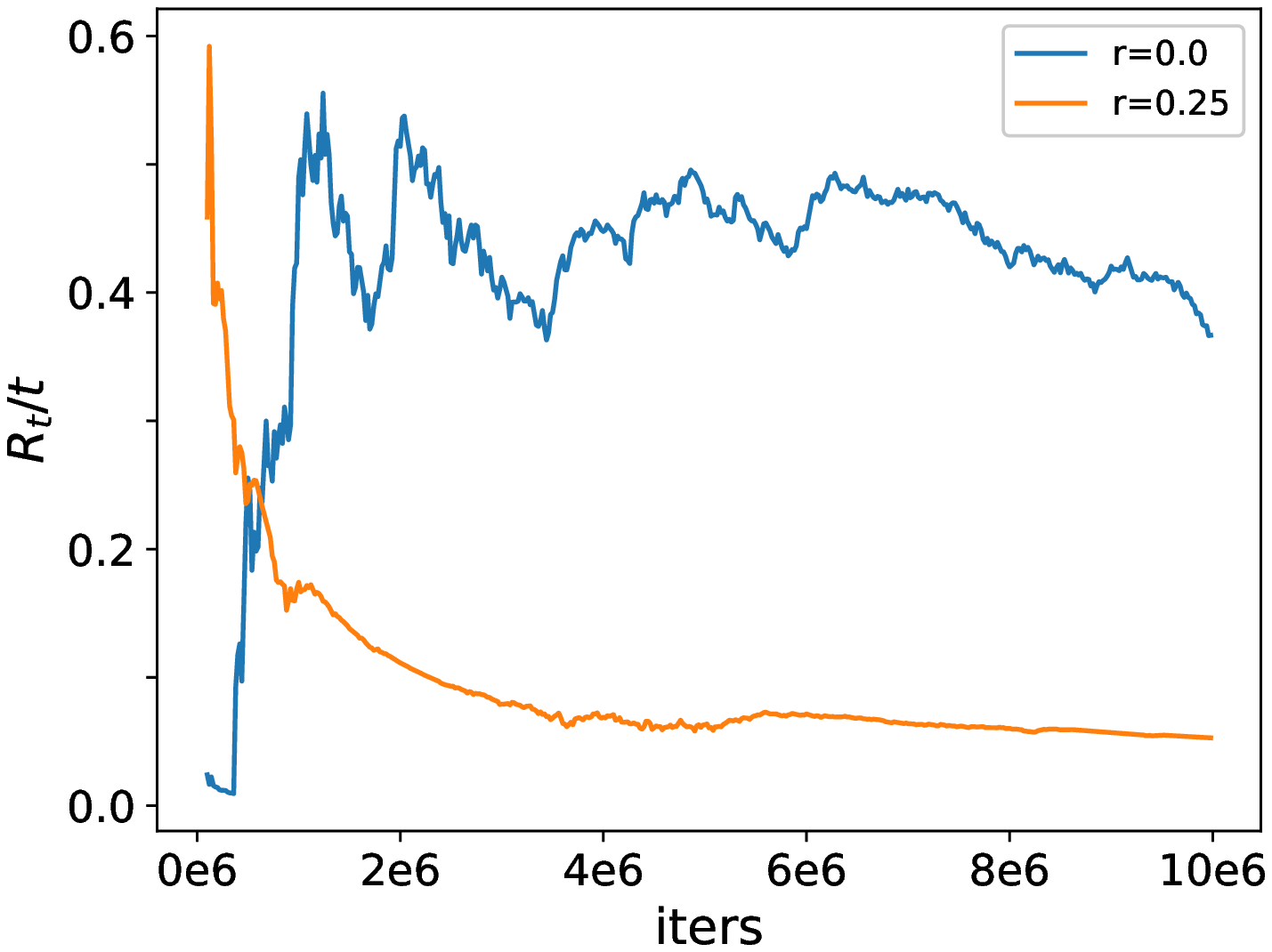}}\label{fig:synthesis_b}
\subfigure[]{\includegraphics[width=0.32\linewidth]{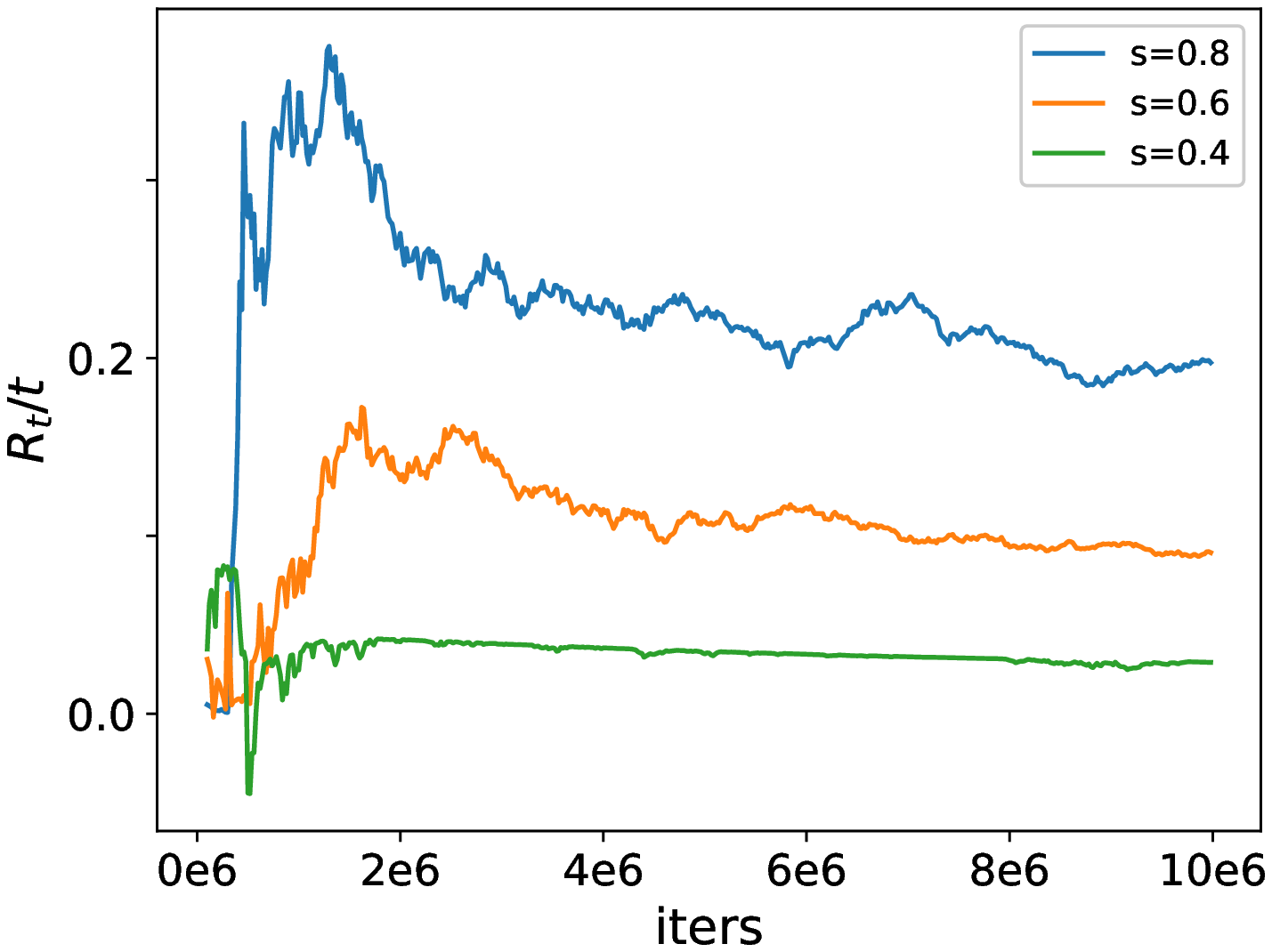}}\label{fig:synthesis_v}
\vspace{-0.3cm}
\\
\centering
\subfigure[]{\includegraphics[width=0.32\linewidth]{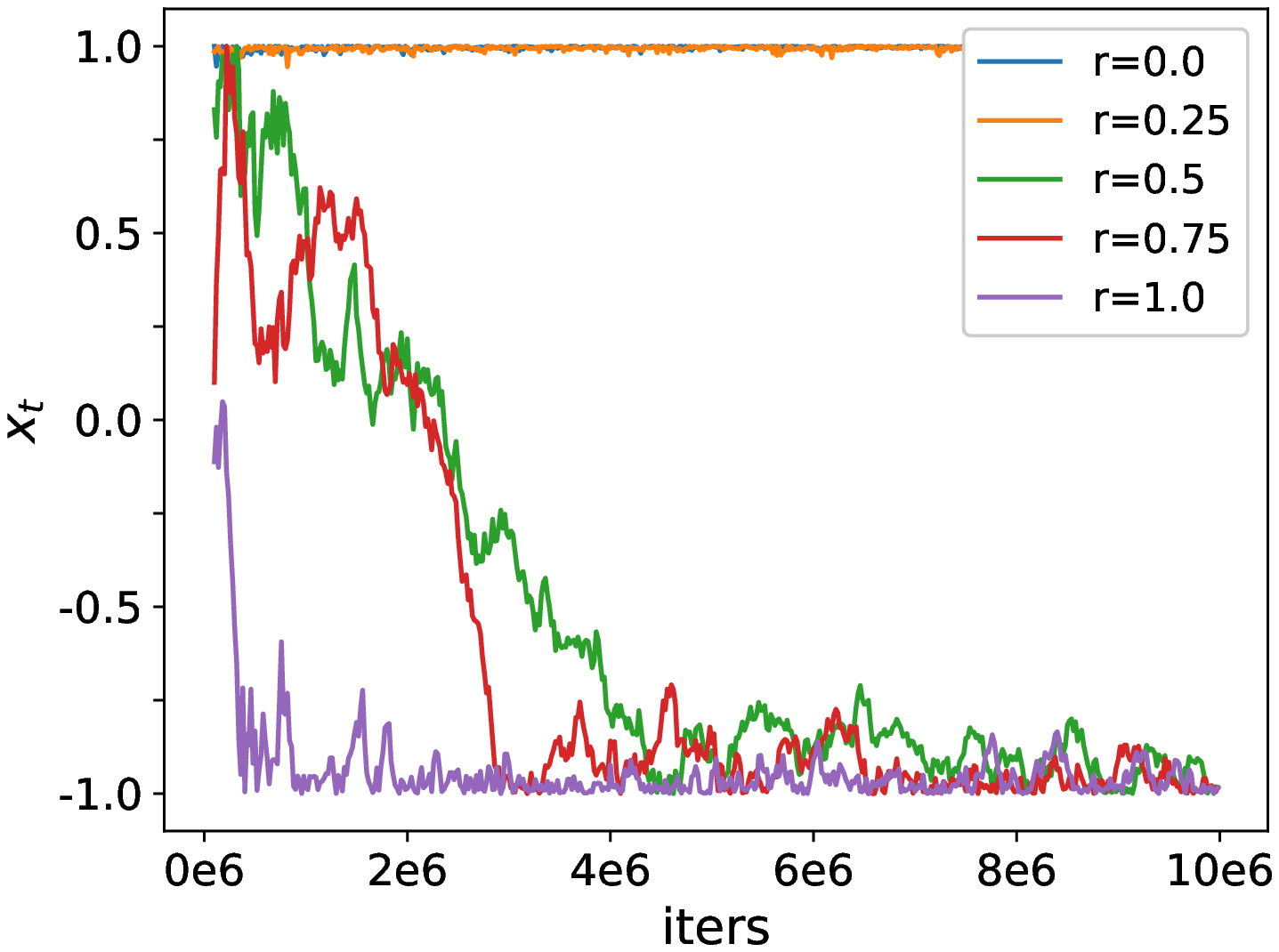}}\label{fig:synthesis_d}
\subfigure[]{\includegraphics[width=0.32\linewidth]{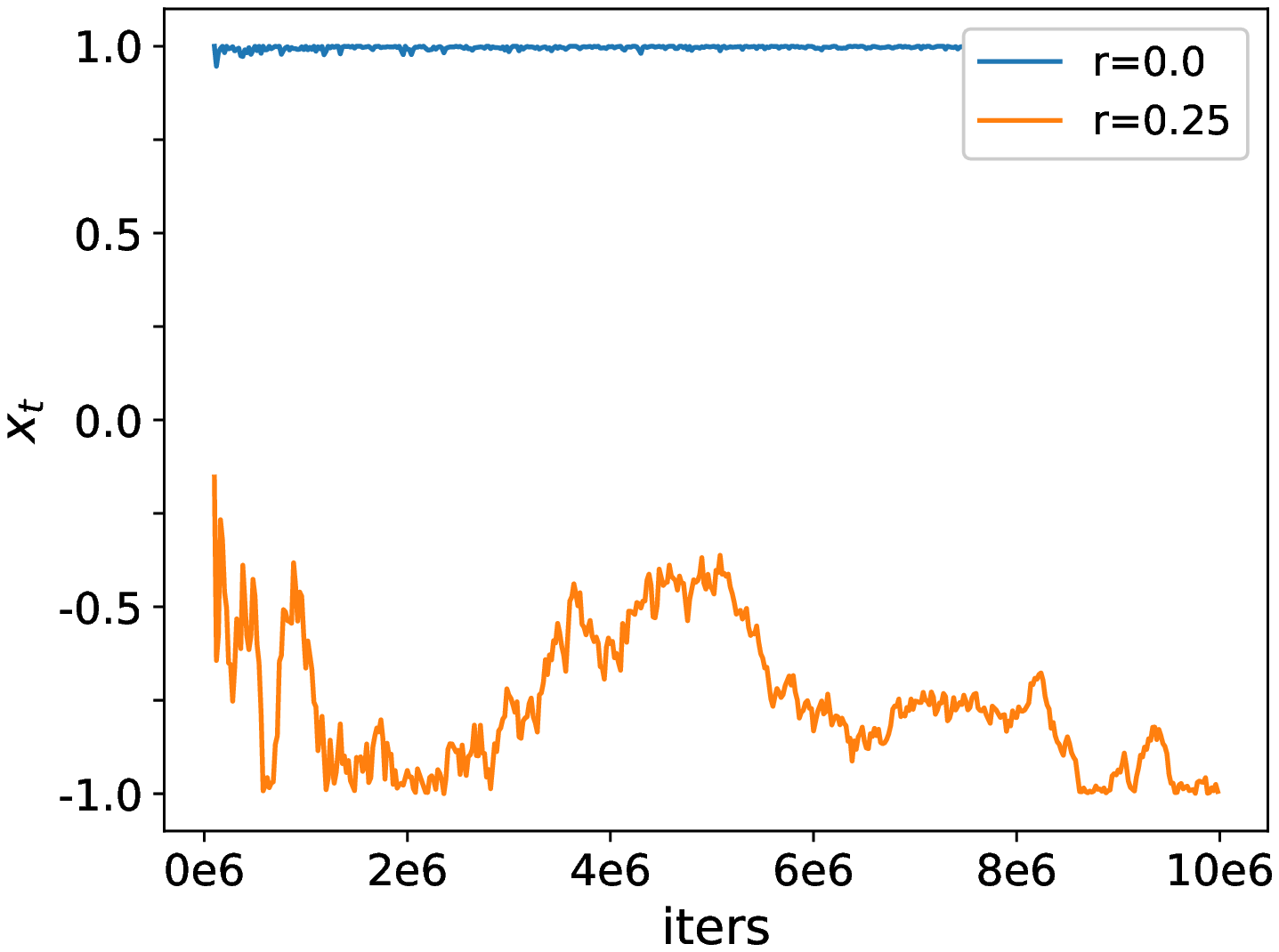}}\label{fig:synthesis_e}
\subfigure[]{\includegraphics[width=0.32\linewidth]{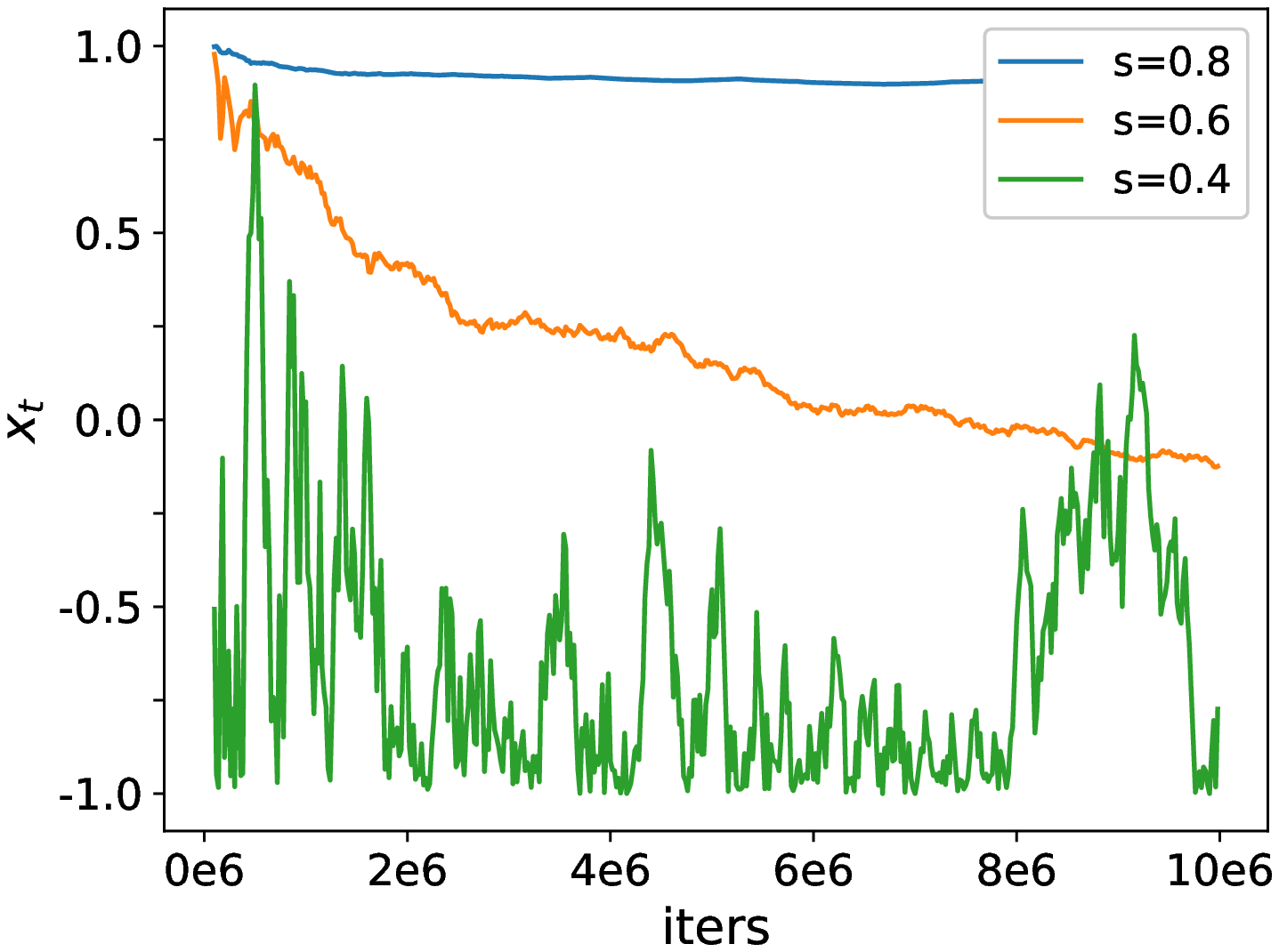}}\label{fig:synthesis_f}
\caption{The above figures are for average regret and $x$ values with different $r$ and $s$ values, respectively.
Figures (a) and (d) plot the performance profiles of Generic Adam with different $r$ values.
Figures (b) and (e) plot the performance profiles of Generic Adam with $\theta_t^{(r)} = 1 - \frac{0.01}{t^r}$ and $r = 0$ and $0.25$.
Figures (c) and (f) plot the performance profiles of Generic Adam with different $s$ values.}
\label{fig:sensivitive}
\vspace{-0.3cm}
\end{figure*}

\smallskip
\noindent
{\bf From the sufficient condition perspective.}\ ~
Let $\alpha_t \!=\! \eta/t^s$ for $0\leq s \!<\! 1$ and $\theta_t \!=\! \theta \!<\! 1$. According to Corollary \ref{constant-theta}, $Bound(T)$ in Theorem \ref{thm1-001} has the following estimate:
\begin{align*}
\frac{C}{\delta \eta T^{1\!-s}}\!+\! \frac{C'\sqrt{1\!-\!\theta}}{\delta}\!\leq\! Bound(T)\! \leq\! \frac{C}{\delta \eta T^{1-s}} \!+\! \frac{C'\sqrt{1\!-\!\theta}}{\delta(1\!-s)}.
\end{align*}
The bounds tell us some points on Adam with constant $\theta_t$:
\begin{itemize}
    \item[1.] $Bound(T)\!=\!\mathcal{O}(1)$, so the convergence is not guaranteed. This result coincides with the divergence issue demonstrated in \cite{Reddi2018on}. Indeed, since in this case Adam is not convergent, this is the best bound we can have.

    \item[2.] Consider the dependence on parameter $s$. The bound is decreasing in $s$. The best bound in this case is when $s=0$, \textit{i.e.}, the base learning rate is taken constant. This explains why in practice taking a more aggressive constant base learning rate often leads to even better performance, comparing with taking a decaying one.

    \item[3.] Consider the dependence on parameter $\theta$. Note that the constants $C$ and $C'$ depend on $\theta_1$ instead of the whole sequence $\theta_t$. We can always set $\theta_t = \theta$ for $t \geq 2$ while fix $\theta_1 < \theta$, by which we can take $C$ and $C'$ independent of constant $\theta$. Then the principal term of $Bound(T)$ is linear in $\sqrt{1-\theta}$, so decreases to zero as $\theta \to 1$. This explains why setting $\theta$ close to 1 in practice often results in better performance in practice.\end{itemize}

Moreover, Corollary \ref{poly-setting} shows us how the convergence rate continuously changes when we continuously verify parameters $\theta_t$. Let us fix $\alpha_t \!=\! 1/\sqrt{t}$ and consider the following continuous family of parameters $\{\theta_t^{(r)}\}$ with $r\in [0, 1]$:
\begin{equation*}
\theta_t^{(r)} = 1 - \alpha^{(r)}/t^r, \text{~~where~~} \alpha^{(r)} = r\bar{\theta} + (1-\bar{\theta}),\ 0 < \bar{\theta} < 1.
\end{equation*}
Note that when $r=1$, then $\theta_t = 1 - 1/t$, this is the AdaEMA, which has the convergence rate $\mathcal{O}(\log T/\sqrt{T})$; when $r = 0$, then $\theta_t = \bar{\theta} < 1$, this is the original Adam with constant $\theta_t$, which only has the $\mathcal{O}(1)$ bound; when $0 < r < 1$, by Corollary \ref{poly-setting}, the algorithm has the $\mathcal{O}(T^{-r/2})$ convergence rate.
Along this continuous family of parameters, we observe that the theoretical convergence rate continuously deteriorates as the real parameter $r$ decreases from 1 to 0, namely, as we gradually shift from AdaEMA to Adam with constant $\theta_t$.
In the limiting case, the latter is not guaranteed with convergence any more.
This phenomenon is empirically verified by the Synthetic Counterexample in Section \ref{experiments}.

\smallskip
\noindent
{\bf From the Weighted AdaEMA perspective.}\ ~Since Generic Adam is equivalent to Weighted AdaEMA, we can examine Adam with $\theta_t = \theta < 1$ in terms of Weighted AdaEMA.
In this case, we find that the associated sequence of weights $w_t = (1 - \theta)\theta^{-t}$ is growing in an exponential order.
Corollary \ref{polyweights} shows that as long as the weights are in polynomial growth, Weighted AdaEMA is convergent and its convergence rate is $\mathcal{O}(\log T/\sqrt{T})$.
This indicates that the exponential-moving-average technique in the estimate of second-order moments may assign a too aggressive weight to the current gradient, which leads to the divergence.

\section{Experiments}\label{experiments}

In this section, we experimentally validate the proposed sufficient condition by applying Generic Adam and RMSProp to solve the counterexample \cite{Reddi2018on} and to train LeNet \cite{lecun1998gradient} on the MNIST dataset \cite{lecun2010mnist} and ResNet \cite{he2016deep} on the CIFAR-100 dataset \cite{krizhevsky2009learning}, respectively.

\subsection{Synthetic Counterexample}

In this experiment, we verify the phenomenon described in Section \ref{insights-for-divergence} that how the convergence rate of Generic Adam gradually changes along a continuous path of families of parameters on the one-dimensional counterexample in \cite{Reddi2018on}:
\begin{equation}\label{counter-example}
\mathcal{R}(T) = \textstyle\sum\limits_{t=1}^{T}f_{t}(x_t) - \min\limits_{x \in [-1, 1]}~ \sum\limits_{t=1}^{T} f_t(x), 
\end{equation}
where $T$ is the number of maximum iterations, $f_{t}(x)\!=\!1010x$ with probability 0.01, and $f_{t}(x)\!=\!10x$ with probability 0.99.

\smallskip
\noindent
{\bf Sensitivity of parameter $r$.}\ ~
We set  $T = 10^7$, $\alpha_t = 0.5 / \sqrt{t}$,
$\beta = 0.9$, and $\theta_t$ as $\theta_t^{(r)} = 1 - (0.01 + 0.99r)/{t^r}$ with $r \in \{0,\ 0.25,\ 0.5,\ 0.75,\ 1.0\}$, respectively.
Note that when $r=0$, Generic Adam reduces to the originally divergent Adam \cite{kingma2014adam} with $(\beta, \bar{\theta}) = (0.9, 0.99)$. When $r=1$, Generic Adam reduces to the AdaEMA \cite{chen2018convergence} with $\beta = 0.9$.

The experimental results are shown in Figures \ref{fig:sensivitive}(a) and \ref{fig:sensivitive}(d). We can see that for $r = 1.0, 0.75$, and $0.5$, Generic Adam is convergent.
Moreover, the convergence becomes slower when $r$ decreases, which exactly matches Corollary \ref{poly-setting}.
On the other hand, for $r = 0$ and $0.25$, Figure \ref{fig:sensivitive}(d) shows that they do not converge. It seems that the divergence for $r = 0.25$ contradicts our theory.
However, this is because when $r$ is very small, the $\mathcal{O}(T^{-r/2})$ convergence rate is so slow that we may not see a convergent trend in even $10^7$ iterations.
Indeed, for $r = 0.25$, we actually have
\[ \theta_t^{(0.25)} \leq 1 - (0.01 + 0.25 * 0.99)/10^{7 * 0.25} \approx 0.9954, \]
which is not sufficiently close to 1.
As a complementary experiment, we fix the numerator and only change $r$ when $r$ is small. We take $\alpha_t$ and $\beta_t$ as the same, while $\theta_t^{(r)} = 1 - \frac{0.01}{t^r}$ for $r = 0$ and $0.25$, respectively.
The result is shown in Figures \ref{fig:sensivitive}(b) and \ref{fig:sensivitive}(e). We can see that Generic Adam with $r= 0.25$ is indeed convergent in this situation.

\begin{figure*}[htpb]
\centering
\subfigure[]{\includegraphics[width=0.32\linewidth]{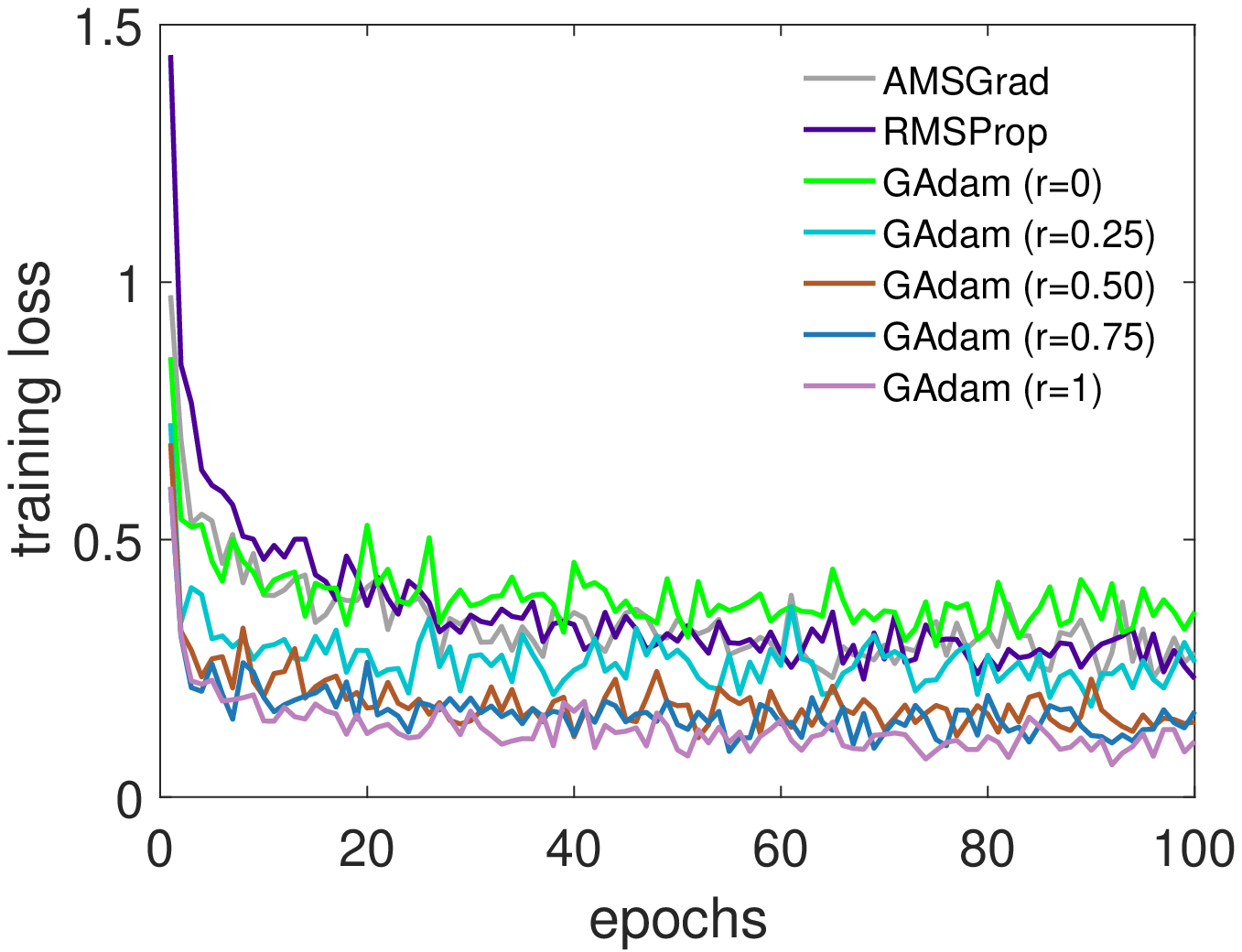}}\label{fig:Lenet_a}
\subfigure[]{\includegraphics[width=0.32\linewidth]{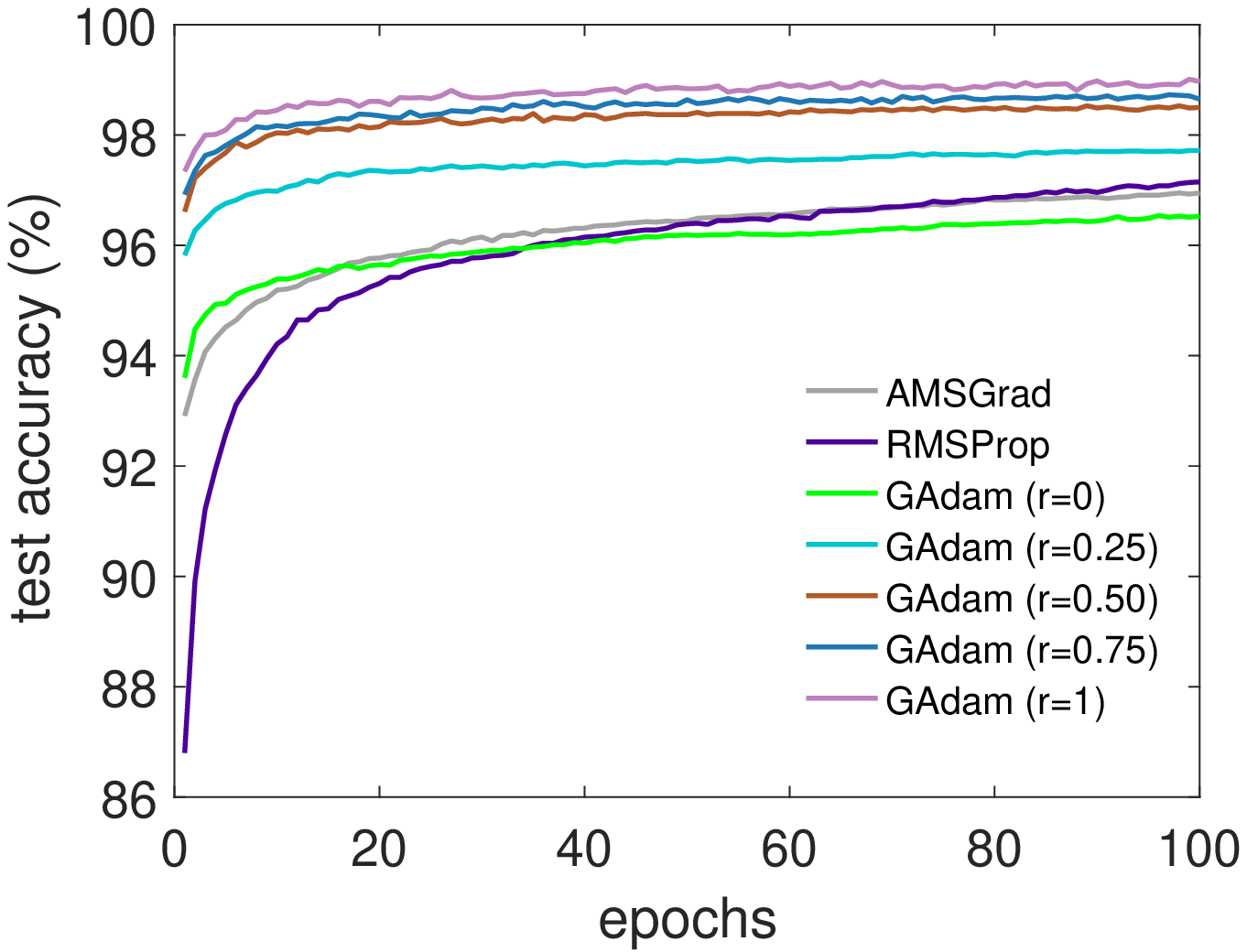}}\label{fig:Lenet_b}
\subfigure[]{\includegraphics[width=0.32\linewidth]{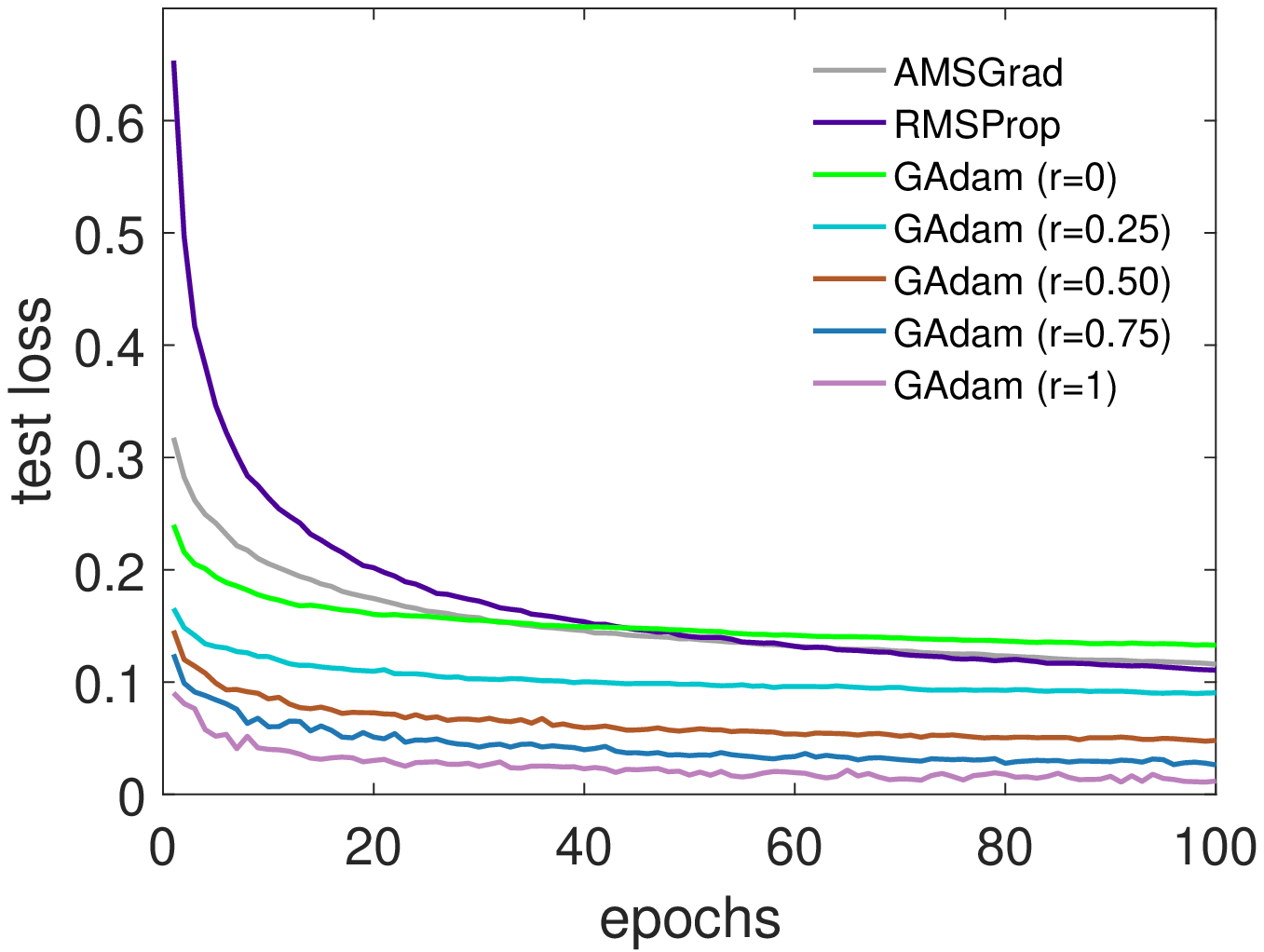}}\label{fig:Lenet_c}
\caption{Performance profiles of Generic Adam with $r=\{0, 0.25, 0.5, 0.75, 1\}$, RMSProp, and AMSGrad for training LeNet on the MNIST dataset. Figures (a), (b), and (c) illustrate training loss vs. epochs, test accuracy vs. epochs, and test loss vs. epochs, respectively.}
\label{fig:LeNet}
\vspace{-0.1cm}
\centering
\subfigure[]{\includegraphics[width=0.32\linewidth]{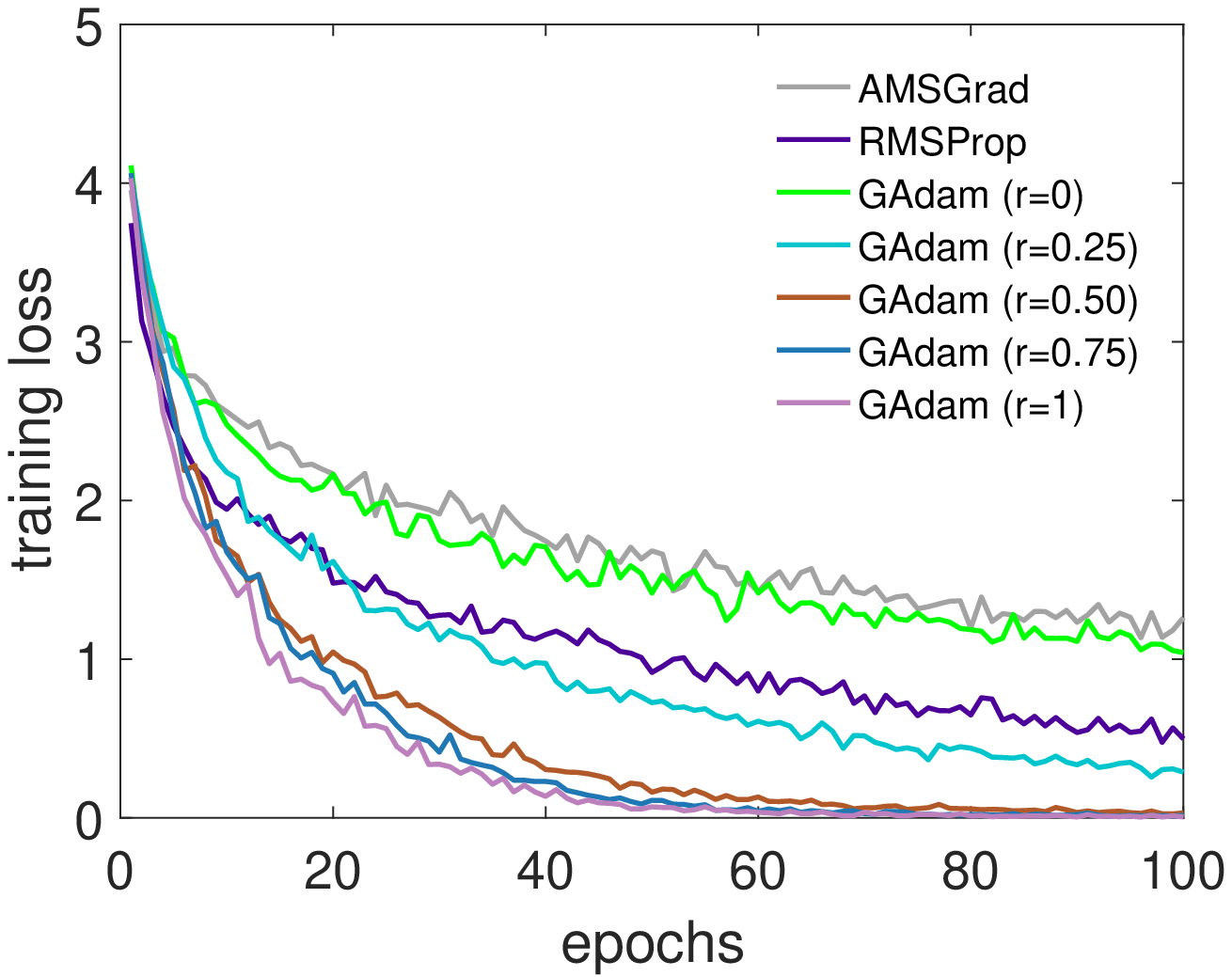}}\label{fig:resnet_a}
\subfigure[]{\includegraphics[width=0.32\linewidth]{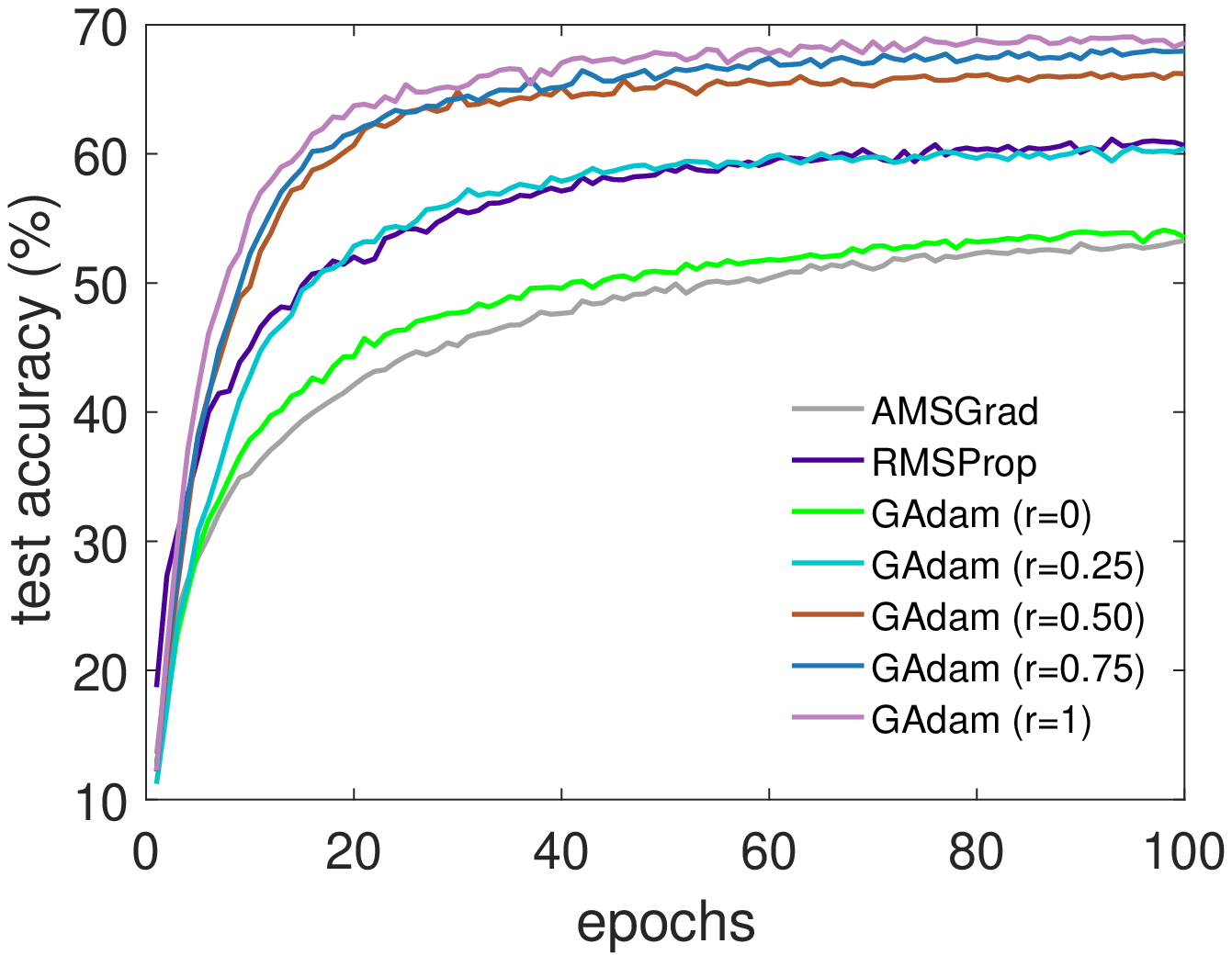}}\label{fig:resnet_b}
\subfigure[]{\includegraphics[width=0.32\linewidth]{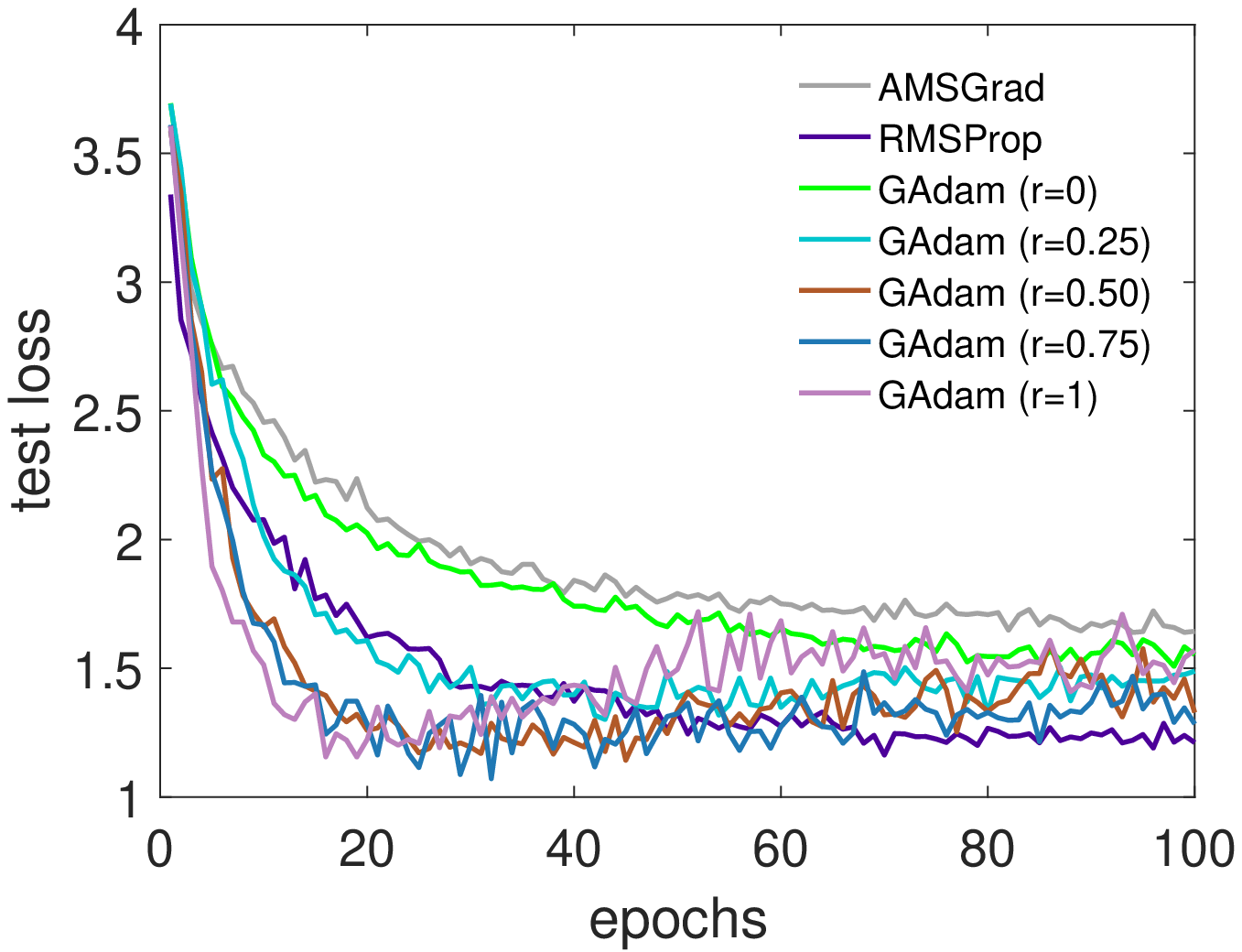}}\label{fig:resnet_c}
\caption{Performance profiles of Generic Adam with $r=\{0, 0.25, 0.5, 0.75, 1\}$, RMSProp, and AMSGrad for training ResNet on the CIFAR-100 dataset. Figures (a)-(c) illustrate training loss vs. epochs, test accuracy vs. epochs, and test loss vs. epochs, respectively.}
\label{fig:ResNet}
\vspace{-0.3cm}
\end{figure*}

\smallskip
\noindent
{\bf Sensitivity of parameter $s$.}\ ~
Now, we show the sensitivity of $s$ of the sufficient condition {(\bf SC)} by fixing $r\!=\!0.8$ and selecting $s$ from the collection $s= \{0.4, 0.6, 0.8\}$.
Figures \ref{fig:sensivitive}(c) and \ref{fig:sensivitive}(e) illustrate the sensitivity of parameter $s$ when Generic Adam is applied to solve the counterexample \eqref{counter-example}. The performance shows that when $s$ is fixed, smaller $r$ can lead to a faster and better convergence speed, which also coincides with the convergence results in Corollary \ref{poly-setting}.

\subsection{LeNet on MNIST and ResNet-18 on CIFAR-100}

In this subsection, we apply Generic Adam 
to train LeNet on the MNIST dataset and ResNet-18 on the CIFAR-100 dataset, respectively, in order to validate the convergence rates in Corollary \ref{poly-setting}.
Meanwhile, the comparisons between Generic Adam and AMSGrad \cite{Reddi2018on,chen2018convergence} are also provided to distinguish their differences in training deep neural networks.
We illustrate the performance profiles in three aspects: training loss vs. epochs, test loss vs. epochs, and test accuracy vs. epochs, respectively.
Besides, the architectures of LeNet and ResNet-18, and the statistics of the MNIST and CIFAR-100 datasets are described in the {\bf supplementary material}.

In the experiments, for Generic Adam, we set $\theta_{t}^{(r)} = 1 - (0.001 + 0.999r)/t^r$ with $r \in \{0, 0.25, 0.5, 0.75,1\}$ and $\beta_t =0.9$, respectively;
for RMSProp, we set $\beta_{t}=0$ and $\theta_{t}=1-\frac{1}{t}$ along with the parameter settings in \cite{mukkamala2017variants}.
For fairness, the base learning rates $\alpha_t$ in Generic Adam, RMSProp, and AMSGrad are all set as $0.001/\sqrt{t}$.
Figures \ref{fig:LeNet} and \ref{fig:ResNet} illustrate the results of Generic Adam with different $r$, RMSProp, and AMSGrad for training  LeNet on MNIST and training ResNet-18 on CIFAR-100, respectively.
We can see that AMSGrad and Adam (Generic Adam with $r=0$) decrease the training loss slowest and show the worst test accuracy among the compared optimizers.
One possible reason is due to the use of constant $\theta$ in AMSGrad and original Adam.
By Figures \ref{fig:LeNet} and \ref{fig:ResNet}, we can observe that the convergences of Generic Adam are extremely sensitive to the choice of parameter $\theta_{t}$.
Larger $r$ can contribute to a faster convergence rate of Generic Adam, which corroborates the theoretical result in Corollary \ref{poly-setting}.
Additionally, the test accuracies in Figures \ref{fig:LeNet}(b) and \ref{fig:ResNet}(b) indicate that a smaller training loss can contribute to a higher test accuracy for Generic Adam.

\section{Conclusions}\label{conclusion}
In this work, we delved into the convergences of Adam and RMSProp, and presented an easy-to-check sufficient condition to guarantee their convergences in the non-convex stochastic setting.
This sufficient condition merely depends on the base learning rate $\alpha_{t}$ and the linear combination parameter $\theta_{t}$ of second-order moments.
Relying on this sufficient condition, we found that the divergences of Adam and RMSProp are possibly due to the incorrect parameter settings of $\alpha_{t}$ and $\theta_{t}$.
In addition, we reformulated Adam as weighted AdaGrad with exponential moving average momentum, which provides a novel perspective for understanding Adam and RMSProp.
At last, the correctness of theoretical results was also verified via the counterexample in \cite{Reddi2018on} and training deep neural networks on real-world datasets.
\newpage
{
\bibliographystyle{ieee}
\bibliography{egbib}
}

\newpage
\newpage
\onecolumn

\vspace{0.5in}
\begin{center}
 \rule{6.875in}{0.7pt}\\ 
 {\Large\bf Supplementary Material for\\ `` A Sufficient Condition for Convergences of Adam and RMSProp ''}
 \rule{6.875in}{0.7pt}
\end{center}
\appendix

In this supplementary we give the complete proofs of our main theoretical results. Section \ref{key-lemmas} introduces the necessary lemmas for the proofs, and Section \ref{main-proof-all} proves the main propositions, theorems, and corollaries. Section \ref{LeNet-and-ResNet} describes the architectures of LeNet and ResNet-18, and the statistics of the training datasets and validation datasets of MNIST and CIFAR-100.

\paragraph{Notations} We use bold letters to represent vectors. The $k$-th component of a vector $\bm{v}_t$ is denoted as ${v}_{t, k}$. The inner product between two vectors $\bm{v}_t$ and $\bm{w}_t$ is denoted as $\langle \bm{v}_t, \bm{w}_t \rangle$. Other than that, all computations that involve vectors shall be understood in the component-wise way. We say a vector $\bm{v}_t \geq 0$ if every component of $\bm{v}_t$ is non-negative, and $\bm{v}_t \geq \bm{w}_t$ if $v_{t,k} \geq w_{t,k}$ for all $k=1, 2, \ldots, d$. The $\ell_1$ norm of a vector $\bm{v}_t$ is defined as $\norm{\bm{v}_t}_1 = \sum_{k=1}^d |{v}_{t, k}|$. The $\ell_2$ norm is defined as $\norm{\bm{v}_t}^2 =\langle \bm{v}_t, \bm{v}_t \rangle = \sum_{k=1}^d |{v}_{t,k}|^2$. Given a positive vector $\bm{\hat{\eta}}_t$, it will be helpful to define the following weighted norm: $\norm{\bm{v}_t}^2_{\bm{\eta}_t} = \langle \bm{v}_t, \bm{\hat{\eta}}_t \bm{v}_t \rangle = \sum_{k=1}^d \hat{\eta}_{t, k}|{v}_{t, k}|^2$.

\medskip

\section{Key Lemmas}\label{key-lemmas}
In this section we provide the necessary lemmas for the proofs of  Theorems \ref{convergence_in_expectation} and \ref{thm1-001}.
\begin{lemma}\label{lem5}
Given $S_0 > 0$ and a non-negative sequence $\{s_t\}$, let $S_t = S_0 + \sum_{i=1}^t s_i$ for $t \geq 1$. Then the following estimate holds
\begin{equation}
\sum_{t=1}^T \frac{s_t}{S_t} \leq \log(S_T) - \log(S_0).
\end{equation}
\end{lemma}
\begin{proof}
The finite sum $\sum_{t=1}^T {s_t}/{S_t}$ can be interpreted as a Riemann sum  $\sum_{t=1}^T (S_t - S_{t-1})/S_t.$
Since $1/x$ is decreasing on the interval $(0, \infty)$, we have
$$\sum_{t=1}^T \frac{S_t-S_{t-1}}{S_t} \leq \int_{S_0}^{S_T} \frac{1}{x} d x = \log(S_T) - \log(S_0).$$
The proof is completed.
\end{proof}

\begin{lemma}[Abel's Lemma - Summation by parts]\label{lem2-003}
Let $\{u_t\}$ and $\{s_t\}$ be two non-negative sequences. Let $S_t = \sum_{i=1}^t s_i$ for $t \geq 1$. Then
\begin{equation}
\sum_{t=1}^T u_t s_t = \sum_{t=1}^{T-1} (u_t - u_{t+1})S_t + u_T S_T.
\end{equation}
\end{lemma}
\begin{proof}
Let $S_0 = 0$. Then
\begin{equation}
\sum_{t=1}^T u_t s_t = \sum_{t=1}^T u_t (S_t - S_{t-1}) = \sum_{t=1}^{T-1} u_tS_t - \sum_{t=1}^{T-1} u_{t+1}S_t + u_TS_T = \sum_{t=1}^{T-1}(u_t - u_{t+1})S_t + u_TS_T.
\end{equation}
The proof is completed.
\end{proof}

\begin{lemma}\label{lem2-002}
Let $\{\theta_t\}$ and $\{\alpha_t\}$ satisfy the restrictions (R2) and (R3). For any $i \leq t$, we have
\begin{equation}
\chi_t \leq C_0 \chi_i \text{~~and~~} \alpha_t \leq C_0 \alpha_i.
\end{equation}

\end{lemma}
\begin{proof}
For any $i \leq t$, since the sequence $\{a_t\}$ is non-increasing, we have $a_t \leq a_i$. Hence,
\begin{equation*}
\chi_t = \frac{\alpha_t}{\sqrt{1-\theta_t}} \leq C_0 a_t \leq C_0 a_i \leq C_0 \frac{\alpha_i}{\sqrt{1-\theta_i}} = C_0 \chi_i,
\end{equation*}
which proves the first inequality. On the other hand, since $\{\theta_t\}$ is non-decreasing, it holds
\begin{equation*}
\alpha_t \leq C_0 \frac{\sqrt{1-\theta_t}}{\sqrt{1-\theta_i}} \alpha_i \leq C_0 \alpha_i = C_0 \alpha_i.
\end{equation*}
The proof is completed.
\end{proof}

\medskip

Let $\Theta_{(t,i)} = \prod_{j=i+1}^t \theta_j$ for $i < t$ and $\Theta_{(t,t)}=1$ by convention. 

\begin{lemma}\label{lem2-001}
Fix a constant $\theta'$ with $\beta^2 < \theta' < \theta$. Let $C_1$ be as given as Eq.~\eqref{Constant-C1} in the main paper. For any $i \leq t$, we have
\begin{equation}
\Theta_{(t,i)} \geq C_1 (\theta')^{t-i}.
\end{equation}
\end{lemma}
\begin{proof}
For any $i \leq t$, since $\theta_j \geq \theta'$ for $j \geq N$, and $\theta_j < \theta'$ for $j < N$, we have
\[\Theta_{(t,i)} = \prod_{j=i+1}^t\theta_j \geq \left(\prod_{j=i+1}^N\theta_j\right)(\theta')^{t-N} = \left(\prod_{j=i+1}^N({\theta_j}/{\theta'})\right)(\theta')^{t-i}
\geq \left(\prod_{j=1}^N ({\theta_j}/{\theta'})\right)(\theta')^{t-i}.\]
We take the constant $C_1 = \prod_{j=1}^N (\theta_j/\theta')$, where $N$ is the maximum of the indices for which $\theta_j < \theta'$. The proof is completed.
\end{proof}
\begin{remark}
If $\theta_t = \theta$ is a constant, we have $\Theta_{(t,i)} = \theta^{t-i}$. In this case we can take $\theta' = \theta$ and $C_1 = 1$.
\end{remark}

\begin{lemma}\label{lem1-001}
Let $\gamma := \beta^2/{\theta'}$. We have the following estimate
\begin{equation}
\bm{m_t}^2 \leq \frac{1}{C_1(1-\gamma)(1-\theta_t)}\bm{v}_t,~\forall t.
\end{equation}
\end{lemma}
\begin{proof}
Let $B_{(t, i)} = \prod_{j = i+1}^t \beta_j$ for $i < t$ and $B_{(t,t)} = 1$ by convention. By the iteration formula $\bm{m}_t = \beta_t \bm{m}_{t-1} + (1-\beta_t)\bm{g}_t$ and $\bm{m}_0 = \bm{0}$, we have
\begin{equation*}
\bm{m}_t = \sum_{i=1}^t \left(\prod_{j=i+1}^t\beta_j\right) (1-\beta_i)\bm{g}_i = \sum_{i=1}^t B_{(t,i)}(1-\beta_i)\bm{g}_i.
\end{equation*}
Similarly, by $\bm{v}_t = \theta_t \bm{v}_{t-1} + (1-\theta_t)\bm{g}_t^2$ and $\bm{v}_0 = \bm{\epsilon}$, we have
\begin{equation*}
\bm{v}_t = \left(\prod_{j=1}^t\theta_j\right)\bm{\epsilon} + \sum_{i=1}^t \left(\prod_{j=i+1}^t \theta_j\right)\left({1-\theta_i}\right) \bm{g}_i^2
\geq \sum_{i=1}^t \Theta_{(t,i)}(1-\theta_i) \bm{g}_i^2.
\end{equation*}
It follows by arithmetic inequality that
\begin{equation*}
\begin{split}
\bm{m}_t^2 &= \left( \sum_{i=1}^t \frac{(1-\beta_i)B_{(t,i)}}{\sqrt{(1-\theta_i)\Theta_{(t,i)}}} \sqrt{(1-\theta_i)\Theta_{(t,i)}} \bm{g}_i\right)^2 \\
&\leq \left(\sum_{i=1}^t \frac{(1 - \beta_i)^2B_{(t,i)}^2}{(1-\theta_i)\Theta_{(t,i)}}\right)
\left(\sum_{i=1}^t \Theta_{(t,i)}(1-\theta_i)\bm{g}_i^2\right)
\leq \left(\sum_{i=1}^t \frac{(1-\beta_i)^2B_{(t,i)}^2}{(1-\theta_i)\Theta_{(t,i)}}\right) \bm{v}_t.
\end{split}
\end{equation*}
Note that $\{\theta_t\}$ is non-decreasing by (\textbf{R}2), and $B_{(t,i)} \leq \beta^{t-i}$ by (\textbf{R}1). By Lemma \ref{lem2-001}, we have
\begin{equation*}\label{1-005}
\sum_{i=1}^t \frac{(1-\beta_i)^2B_{(t,i)}^2}{(1-\theta_i)\Theta_{(t,i)}}
\leq \frac{1}{C_1(1-\theta_t)}\sum_{i=1}^t \left(\frac{\beta^2}{\theta'}\right)^{t-i} \leq \frac{1}{C_1(1-\theta_t)}\sum_{k=0}^{t-1} \gamma^k
\leq \frac{1}{C_1(1-\gamma)(1-\theta_t)}.
\end{equation*}
The proof is completed.
\end{proof}


\medskip

Let $\bm{\Delta}_t := \bm{x}_{t+1} - \bm{x}_t = -\alpha_t\bm{m}_t/\sqrt{\bm{v}_t}$. Let $\bm{\hat{v}}_t = \theta_t \bm{v}_{t-1} + (1-\theta_t) \bm{\sigma_t}^2$, where $\bm{\sigma_t}^2 = \mathbb{E}_t \left[\bm{g}_t^2\right]$ and let $\bm{\hat{\eta}_t} = \alpha_t/\sqrt{\bm{\hat{v}_t}}$. 

\begin{lemma}\label{lem1-002}
The following equality holds
\begin{equation}\label{1-008}
\begin{split}
\bm{\Delta}_t - \frac{\beta_t\alpha_t}{\sqrt{\theta_t}\alpha_{t-1}} \bm{\Delta}_{t-1}
= -(1 - \beta_t)\bm{\hat{\eta}}_t\bm{g}_t + \bm{\hat{\eta}}_t\bm{g}_t\frac{(1-\theta_t)\bm{g}_t}{\sqrt{\bm{v}}_t }\bm{A}_t
+ \bm{\hat{\eta}}_t\bm{\sigma}_t\frac{(1-\theta_t)\bm{g}_t}{\sqrt{\bm{v}_t}}\bm{B}_t,
\end{split}
\end{equation}
where
\begin{equation*}
\begin{split}
\bm{A}_t &= \frac{\beta_t\bm{m}_{t-1}}{\sqrt{\bm{v}_t}+\sqrt{\theta_t\bm{v}_{t-1}}}+\frac{(1-\beta_t)\bm{g}_t}{\sqrt{\bm{v}}_t+\sqrt{\bm{\hat{v}}_t}},\\
\bm{B}_t &= \left(\frac{\beta_t\bm{m}_{t-1}}{\sqrt{\theta_t\bm{v}_{t-1}}}\frac{\sqrt{1-\theta_t}\bm{g}_t}{\sqrt{\bm{v}}_t + \sqrt{\theta_t\bm{v}_{t-1}}}\frac{\sqrt{1-\theta_t}\bm{\sigma}_t}{\sqrt{\bm{\hat{v}}_t}+\sqrt{\theta_t\bm{v}_{t-1}}}\right)
- \frac{(1-\beta_t)\bm{\sigma}_t}{\sqrt{\bm{v}_t}+\sqrt{\bm{\hat{v}}_t}}.
\end{split}
\end{equation*}
\end{lemma}
\begin{proof}
We have
\begin{equation}
\begin{split}
\bm{\Delta}_t - \frac{\beta_t\alpha_t}{\sqrt{\theta_t}\alpha_{t-1}} \bm{\Delta}_{t-1} =~& -\frac{\alpha_t\bm{m_t}}{\sqrt{\bm{v}_t}} + \frac{\beta_t\alpha_{t}\bm{m}_{t-1}}{\sqrt{\theta_t\bm{v}_{t-1}}}
= -\alpha_t \left(\frac{\bm{m}_t}{\sqrt{\bm{v}_t}} - \frac{\beta_t \bm{m}_{t-1}}{\sqrt{\theta_t \bm{v}_{t-1}}}\right) \\
=~& -\underbrace{\frac{(1-\beta_t)\alpha_t \bm{g}_t}{\sqrt{\bm{v}_t}}}_{\text{(I)}} + \underbrace{\beta_t \alpha_t \bm{m}_{t-1} \left( \frac{1}{\sqrt{\theta_t \bm{v}_{t-1}}} - \frac{1}{\sqrt{\bm{v}_t}}\right)}_{\text{(II)}}. 
\end{split}
\end{equation}
For (I) we have
\begin{equation}\label{1-010}
\begin{split}
\text{(I)} =~& \frac{(1-\beta_t)\alpha_t\bm{g}_t}{\sqrt{\bm{\hat{v}}}_t} + (1-\beta_t)\alpha_t\bm{g}_t\left(\frac{1}{\sqrt{\bm{v}}_t}- \frac{1}{\sqrt{\bm{\hat{v}}_t}}\right)\\
=~& (1-\beta_t)\bm{\hat{\eta}}_t\bm{g}_t + (1-\beta_t)\alpha_t\bm{g}_t\frac{(1-\theta_t)(\bm{\sigma}_t^2-\bm{g}_t^2)}{\sqrt{\bm{v}_t}\sqrt{\bm{\hat{v}}_t}(\sqrt{\bm{v}_t}+\sqrt{\bm{\hat{v}}_t})}\\
=~& (1-\beta_t)\bm{\hat{\eta}}_t\bm{g}_t + \bm{\hat{\eta}}_t\bm{\sigma}_t\frac{(1-\theta_t)\bm{g}_t}{\sqrt{\bm{v}_t}}\frac{(1-\beta_t)\bm{\sigma}_t}{\sqrt{\bm{v}_t}+\sqrt{\bm{\hat{v}}_t}} -
\bm{\hat{\eta}}_t\bm{g}_t\frac{(1-\theta_t)\bm{g}_t}{\sqrt{\bm{v}_t}}\frac{(1-\beta_t)\bm{g}_t}{\sqrt{\bm{v}_t}+\sqrt{\bm{\hat{v}}}_t}.
\end{split}
\end{equation}
For (II) we have
\begin{equation}\label{1-011}
\begin{split}
\text{(II)} &= \beta_t \alpha_t \bm{m}_{t-1} \frac{(1-\theta_t)\bm{g}_t^2}{\sqrt{\bm{v}_t}\sqrt{\theta_t \bm{v}_{t-1}}(\sqrt{\bm{v}_t} + \sqrt{\theta_t \bm{v}_{t-1}})} \\
&= \beta_t \alpha_t \bm{m}_{t-1} \frac{(1-\theta_t)\bm{g}_t^2}{\sqrt{\bm{v}_t}\sqrt{\bm{\hat{v}}_t} (\sqrt{\bm{v}_t} + \sqrt{\theta_t \bm{v}_{t-1}})}
+ \beta_t \alpha_t \bm{m}_{t-1} \frac{(1-\theta_t)\bm{g}_t^2}{\sqrt{\bm{v}_t}(\sqrt{\bm{v}_t} + \sqrt{\theta_t \bm{v}_{t-1}})}\left(\frac{1}{\sqrt{\theta_t\bm{v}_{t-1}}}-\frac{1}{\sqrt{\bm{\hat{v}}_t}}\right)\\
&= \bm{\hat{\eta}}_t \bm{g}_t \frac{(1-\theta_t)\bm{g}_t}{\sqrt{\bm{v}_t}}\left(\frac{\beta_t\bm{m}_{t-1}}{\sqrt{\bm{v}_{t}}+\sqrt{\theta_t \bm{v}_{t-1}}}\right) + \frac{\beta_t \alpha_t \bm{m}_{t-1}(1-\theta_t)^2\bm{g}_t^2\bm{\sigma}_t^2}{\sqrt{\bm{v}_t}\sqrt{\bm{\hat{v}}_t}\sqrt{\theta_t \bm{v}_{t-1}}(\sqrt{\bm{v}_t} + \sqrt{\theta_t \bm{v}_{t-1}})(\sqrt{\bm{\hat{v}}_t}+\sqrt{\theta_t \bm{v}_{t-1}})} \\
&=  \bm{\hat{\eta}}_t \bm{g}_t \frac{(1-\theta_t)\bm{g}_t}{\sqrt{\bm{v}_t}}\left(\frac{\beta_t\bm{m}_{t-1}}{\sqrt{\bm{v}_{t}}+\sqrt{\theta_t \bm{v}_{t-1}}}\right)  +  \bm{\hat{\eta}}_t \bm{\sigma}_t \frac{(1-\theta_t)\bm{g}_t}{\sqrt{\bm{v}_t}}\left(\frac{\beta_t\bm{m}_{t-1}}{\sqrt{\theta_t\bm{v}_{t-1}}}\frac{\sqrt{1-\theta_t}\bm{g}_t}{\sqrt{\bm{v}}_t + \sqrt{\theta_t\bm{v}_{t-1}}}\frac{\sqrt{1-\theta_t}\bm{\sigma}_t}{\sqrt{\bm{\hat{v}}_t}+\sqrt{\theta_t\bm{v}_{t-1}}}\right).
\end{split}
\end{equation}
Combining Eq.~\eqref{1-010} and Eq.~\eqref{1-011}, we obtain the desired Eq.~\eqref{1-008}. The proof is completed.
\end{proof}



\medskip

\begin{lemma} \label{lem1-004}
Let $M_t = \mathbb{E} \left[\left\langle \bm{\nabla} f(\bm{x}_{t}), \bm{\Delta}_{t}\right\rangle + L\norm{\bm{\Delta}_{t}}^2\right]$ and $\chi_t = {\alpha_t}/{\sqrt{1-\theta_t}}$. Then for any $t\geq 2$, we have
\begin{equation}\label{2-012}
\begin{split}
M_t \leq & \frac{\beta_t\alpha_t}{\sqrt{\theta_t}\alpha_{t-1}} M_{t-1} + L\ \mathbb{E}\left[\norm{\bm{\Delta}_t}^2\right]
+ C_2G\chi_t \mathbb{E}\left[\norm{\frac{\sqrt{1-\theta_t}\bm{g}_t}{\sqrt{\bm{v}_t}}}^2\right]
- \frac{1-\beta}{2}\mathbb{E}\left[\norm{\bm{\nabla} f(\bm{x}_t)}^2_{\bm{\hat{\eta}}_t}\right]
\end{split}
\end{equation}
and
\begin{equation}\label{2-013}
M_1 \leq L\ \mathbb{E}\left[\norm{\bm{\Delta}_1}^2\right] + C_2G\chi_1 \mathbb{E}\left[\norm{\frac{\sqrt{1-\theta_t}\bm{g}_1}{\sqrt{v}_1}}^2\right],
\end{equation}
where $C_2 = 2\left(\frac{\beta/(1-\beta)}{\sqrt{C_1(1-\gamma)\theta_1}}+1\right)^2$.
\end{lemma}
\begin{proof}
First, for $t \geq 2$ we have
\begin{equation}\label{1-013}
\begin{split}
\mathbb{E}\langle \bm{\nabla} f(\bm{x}_t), \bm{\Delta}_t \rangle
= \underbrace{\frac{\beta_t\alpha_t}{\sqrt{\theta_t}\alpha_{t-1}} \mathbb{E}\langle \bm{\nabla} f(\bm{x}_t), \bm{\Delta}_{t-1} \rangle}_{\text{(I)}} +
\underbrace{\mathbb{E}\left\langle \bm{\nabla} f(\bm{x}_t), \bm{\Delta}_t - \frac{\beta_t \alpha_t}{\sqrt{\theta_t}\alpha_{t-1}}\bm{\Delta}_{t-1}\right\rangle}_{\text{(II)}}.
\end{split}
\end{equation}
To estimate (I), by the Schwartz inequality and the Lipschitz continuity of the gradient of $f$, we have
\begin{equation}\label{1-015}
\begin{split}
\langle \bm{\nabla} f(\bm{x}_t), \bm{\Delta}_{t-1}\rangle
\leq~&  \langle\bm{\nabla} f(\bm{x}_{t-1}), \bm{\Delta}_{t-1}\rangle + \langle \bm{\nabla} f(\bm{x}_t) - \bm{\nabla} f(\bm{x}_{t-1}), \bm{\Delta}_{t-1}\rangle \\
\leq~&  \langle\bm{\nabla} f(\bm{x}_{t-1}), \bm{\Delta}_{t-1}\rangle +  L \norm{\bm{x}_t - \bm{x}_{t-1}}\norm{\bm{\Delta}_{t-1}} \\
=~&  \langle\bm{\nabla} f(\bm{x}_{t-1}), \bm{\Delta}_{t-1}\rangle + L \norm{\bm{\Delta}_{t-1}}^2. \\
\end{split}
\end{equation}
Hence, we have
\begin{equation}
\text{(I)} \leq \frac{\beta_t\alpha_t}{\sqrt{\theta_t}\alpha_{t-1}} \mathbb{E}\left[\langle\bm{\nabla} f(\bm{x}_{t-1}), \bm{\Delta}_{t-1}\rangle + L \norm{\bm{\Delta}_{t-1}}^2\right] = \frac{\beta_t\alpha_t}{\sqrt{\theta_t}\alpha_{t-1}} M_{t-1}.
\end{equation}
To estimate (II), by Lemma \ref{lem1-002}, we have
\begin{equation}\label{1-016}
\begin{split}
& \mathbb{E}\left\langle \bm{\nabla} f(\bm{x}_t), \bm{\Delta}_t - \frac{\beta_t\alpha_t}{\sqrt{\theta_t}\alpha_{t-1}}\bm{\Delta}_{t-1}\right\rangle \\
=& -(1-\beta_t)\mathbb{E}\langle \bm{\nabla} f(\bm{x}_t), \bm{\hat{\eta}}_t\bm{g}_t \rangle
    \underbrace{- \mathbb{E}\left\langle \bm{\nabla} f(\bm{x}_t), \bm{\hat{\eta}}_t\bm{g}_t\frac{(1-\theta_t)\bm{g}_t}{\sqrt{\bm{v}}_t}\bm{A}_t \right\rangle}_{\text{(III)}}
   \underbrace{- \mathbb{E}\left\langle \bm{\nabla} f(\bm{x}_t), \bm{\hat{\eta}}_t\bm{\sigma}_t\frac{(1-\theta_t)\bm{g}_t}{\sqrt{\bm{v}}_t}\bm{B}_t \right\rangle}_{\text{(IV)}}.
\end{split}
\end{equation}
Note that $\bm{\hat{\eta}}_t$ is independent of $\bm{g}_t$ and $\mathbb{E}_t[\bm{g}_t] = \bm{\nabla} f(\bm{x}_t)$. Hence, for the first term in the right hand side of Eq.~\eqref{1-016}, we have
\begin{equation}\label{1-017}
\begin{split}
-(1-\beta_t)\mathbb{E}\langle \bm{\nabla} f(\bm{x}_t), \bm{\hat{\eta}}_t\bm{g}_t \rangle
&= -(1-\beta_t)\mathbb{E}\left\langle \bm{\nabla} f(\bm{x}_t), \bm{\hat{\eta}}_t\mathbb{E}_t[\bm{g}_t] \right\rangle \\
&= -(1-\beta_t)\mathbb{E}\norm{\bm{\nabla} f(\bm{x}_t)}^2_{\bm{\hat{\eta}}_t}\\
&\leq -(1-\beta)\mathbb{E}\norm{\bm{\nabla} f(\bm{x}_t)}^2_{\bm{\hat{\eta}}_t}.
\end{split}
\end{equation}
To estimate (III), we have
\begin{equation}\label{1-018}
\begin{split}
\text{(III)}~\leq \mathbb{E}\left\langle \frac{\sqrt{\bm{\hat{\eta}}_t}|\bm{\nabla} f(\bm{x}_t)||\bm{g}_t|}{\bm{\sigma}_t}, \frac{\sqrt{\bm{\hat{\eta}}_t}\bm{\sigma}_t|\bm{A}_t|(1-\theta_t)|\bm{g}_t|}{\sqrt{\bm{v}_t}}\right\rangle.
\end{split}
\end{equation}
Note that $\bm{\sigma}_t \leq G$. Therefore,
\begin{equation}\label{1-023}
\sqrt{\bm{\hat{\eta}}_t} \bm{\sigma}_t = \sqrt{\bm{\hat{\eta}}_t\bm{\sigma}_t^2} = \sqrt{\frac{\alpha_t \bm{\sigma}_t^2}{\sqrt{\bm{\hat{v}}_t}}}
\leq \sqrt{\frac{\alpha_t \bm{\sigma}_t^2}{\sqrt{(1-\theta_t)\bm{\sigma}_t^2}}} \leq \sqrt{\frac{G\alpha_t}{\sqrt{1-\theta_t}}} = \sqrt{G\chi_t}.
\end{equation}
On the other hand,
\begin{equation}
\begin{split}
|\bm{A}_t|
= \left|\frac{\beta_t\bm{m}_{t-1}}{\sqrt{\bm{v}_t}+\sqrt{\theta_t\bm{v}_{t-1}}}+\frac{(1-\beta_t)\bm{g}_t}{\sqrt{\bm{v}}_t+\sqrt{\bm{\hat{v}}_t}}\right|
\leq \frac{\beta_t|\bm{m}_{t-1}|}{\sqrt{\theta_t \bm{v}_{t-1}}} + \frac{(1-\beta_t)|\bm{g}_t|}{\sqrt{\bm{v}_t}}.
\end{split}
\end{equation}
By Lemma \ref{lem1-001}, we have
\begin{equation}
\frac{|\bm{m}_{t-1}|}{\sqrt{\bm{v}_{t-1}}} \leq \frac{1}{\sqrt{C_1(1-\gamma)(1-\theta_t)}}.
\end{equation}
Meanwhile,
\begin{equation} \frac{|\bm{g}_t|}{\sqrt{\bm{v}_t}} \leq \frac{|\bm{g}_t|}{\sqrt{(1-\theta_t)\bm{g}_t^2}} = \frac{1}{\sqrt{1-\theta_t}}.
\end{equation}
Hence, we have
\begin{equation}
\begin{split}
|\bm{A}_t| &\leq \frac{\beta_t}{\sqrt{C_1(1-\gamma)(1-\theta_t)\theta_t}} + \frac{1-\beta_t}{\sqrt{1-\theta_t}}
\leq \left(\frac{\beta_t/(1-\beta_t)}{\sqrt{C_1(1-\gamma)\theta_t}}+1\right)\frac{1-\beta_t}{\sqrt{1-\theta_t}} \\
&\leq \left(\frac{\beta/(1-\beta)}{\sqrt{C_1(1-\gamma)\theta_1}} + 1\right)\frac{1-\beta_t}{\sqrt{1-\theta_t}} := \frac{C_2'(1-\beta_t)}{\sqrt{1-\theta_t}},
\end{split}
\end{equation}
where $C_2'= \left(\frac{\beta/(1-\beta)}{\sqrt{C_1(1-\gamma)\theta_1}} + 1\right)$. The last inequality holds due to $\beta_t/(1-\beta_t) \leq \beta/(1-\beta)$ as $\beta_t \leq \beta$.
Therefore, we have
\begin{equation}\label{1-021}
\begin{split}
&\left\langle \frac{\sqrt{\bm{\hat{\eta}}_t}|\bm{\nabla} f(\bm{x}_t)||\bm{g}_t|}{\bm{\sigma}_t}, \frac{\sqrt{\bm{\hat{\eta}}_t}\bm{\sigma}_t|\bm{A}_t|(1-\theta_t)|\bm{g}_t|}{\sqrt{\bm{v}_t}}\right\rangle \\
\leq~& \left\langle \frac{\sqrt{\bm{\hat{\eta}}_t}|\bm{\nabla} f(\bm{x}_t)||\bm{g}_t|}{\bm{\sigma}_t},
\sqrt{G\chi_t}C_2'(1-\beta_t)\frac{\sqrt{1-\theta_t}|\bm{g}_t|}{\sqrt{\bm{v}_t}}\right\rangle \\
\leq~& \frac{1-\beta_t}{4}\norm{\frac{\sqrt{\bm{\hat{\eta}}_t}|\bm{\nabla} f(\bm{x}_t)||\bm{g}_t|}{\bm{\sigma}_t}}^2 +
{C_2'^2G}{(1-\beta_t)}\chi_t\norm{\frac{\sqrt{1-\theta_t}\bm{g}_t}{\sqrt{\bm{v}_t}}}^2 \\
\leq~& \frac{1-\beta_t}{4}\norm{\frac{\bm{\hat{\eta}}_t|\bm{\nabla} f(\bm{x}_t)|^2|\bm{g}_t|^2}{\bm{\sigma}_t^2}}_1 +
{C_2'^2G}\chi_t\norm{\frac{\sqrt{1-\theta_t}\bm{g}_t}{\sqrt{\bm{v}_t}}}^2.
\end{split}
\end{equation}
Note that $\bm{\sigma}_t^2 = \mathbb{E}_t[\bm{g}_t^2]$. Hence,
\begin{equation}\label{1-022}
\mathbb{E}_t\norm{\frac{\bm{\hat{\eta}}_t|\bm{\nabla} f(\bm{x}_t)|^2|\bm{g}_t|^2}{\bm{\sigma}_t^2}}_1
= \norm{\bm{\hat{\eta}}_t|\bm{\nabla} f(\bm{x}_t)|^2}_1 = \norm{\bm{\nabla} f(\bm{x}_t)}_{\bm{\hat{\eta}}_t}^2.
\end{equation}
Combining Eq.~\eqref{1-018}, Eq.~\eqref{1-021}, and Eq.~\eqref{1-022}, we obtain
\begin{equation}\label{1-024}
\begin{split}
\text{(III)} 
\leq~& \frac{1-\beta_t}{4}\mathbb{E}\left[\norm{\bm{\nabla} f(\bm{x}_t)}_{\bm{\hat{\eta}}_t}^2\right] +
{C_2'^2G}\chi_t\mathbb{E}\norm{\frac{\sqrt{1-\theta_t}\bm{g}_t}{\sqrt{\bm{v}_t}}}^2.
\end{split}
\end{equation}
The term (IV) is estimated similarly as term (III). First, we have
\begin{equation}
\begin{split}
|\bm{B}_t| \leq~& \left(\frac{\beta_t|\bm{m}_{t-1}|}{\sqrt{\theta_t\bm{v}_{t-1}}}\frac{\sqrt{1-\theta_t}|\bm{g}_t|}{\sqrt{\bm{v}}_t + \sqrt{\theta_t\bm{v}_{t-1}}}\frac{\sqrt{1-\theta_t}\bm{\sigma}_t}{\sqrt{\bm{\hat{v}}_t}+\sqrt{\theta_t\bm{v}_{t-1}}}\right)
+ \frac{(1-\beta_t)\bm{\sigma}_t}{\sqrt{\bm{v}_t}+\sqrt{\bm{\hat{v}}_t}} \\
\leq~& \left(\frac{\beta/(1-\beta)}{\sqrt{C_1(1-\gamma)\theta_1}}+1\right)\frac{1-\beta_t}{\sqrt{1-\theta_t}}
= \frac{C_2'(1-\beta_t)}{\sqrt{1-\theta_t}},
\end{split}
\end{equation}
where $C_2'$ is the constant defined above. We have
\begin{equation}\label{1-026}
\begin{split}
\text{(IV)} \leq~& \mathbb{E}\left\langle \sqrt{\bm{\hat{\eta}}_t}|\bm{\nabla} f(\bm{x}_t)|, \frac{\sqrt{\bm{\hat{\eta}}_t}\bm{\sigma}_t|\bm{B}_t|(1-\theta_t)|\bm{g}_t|}{\sqrt{\bm{v}_t}}\right\rangle \\
\leq~& \mathbb{E}\left\langle \sqrt{\bm{\hat{\eta}}_t}|\bm{\nabla} f(\bm{x}_t)|,
{\sqrt{G\chi_t}C_2'(1-\beta_t)}\frac{\sqrt{1-\theta_t}|\bm{g}_t|}{\sqrt{\bm{v}_t}}\right\rangle \\
\leq~& \frac{1-\beta_t}{4}\mathbb{E}\left[\norm{\bm{\nabla} f(\bm{x}_t)}_{\bm{\hat{\eta}}_t}^2\right] + {C_2'^2G}{\chi_t}\mathbb{E}\norm{\frac{\sqrt{1-\theta_t}\bm{g}_t}{\sqrt{\bm{v}_t}}}^2. \\
\end{split}
\end{equation}
Combining Eq.~\eqref{1-013}, Eq.~\eqref{1-015}, Eq.~\eqref{1-016}, Eq.~\eqref{1-017}, Eq.~\eqref{1-024}, and Eq.~\eqref{1-026}, 
we obtain
\begin{equation}\label{1-030}
\begin{split}
\mathbb{E}\langle \bm{\nabla} f(\bm{x}_t), \bm{\Delta}_t \rangle
&~\leq \frac{\beta_t\alpha_t}{\sqrt{\theta_t}\alpha_{t-1}} M_{t-1}
+ {2C_2'^2G}{\chi_t}\mathbb{E}\norm{\frac{\sqrt{1-\theta_t}\bm{g}_t}{\sqrt{\bm{v}_t}}}^2
- \frac{1-\beta_t}{2}\mathbb{E}\left[\norm{\bm{\nabla} f(\bm{x}_t)}_{\bm{\hat{\eta}}_t}^2\right]\\
&~\leq \frac{\beta_t\alpha_t}{\sqrt{\theta_t}\alpha_{t-1}} M_{t-1}
+ {2C_2'^2G}{\chi_t}\mathbb{E}\norm{\frac{\sqrt{1-\theta_t}\bm{g}_t}{\sqrt{\bm{v}_t}}}^2
- \frac{1-\beta}{2}\mathbb{E}\left[\norm{\bm{\nabla} f(\bm{x}_t)}_{\bm{\hat{\eta}}_t}^2\right].
\end{split}
\end{equation}
Let $C_2$ denote the constant ${2(C_2')^2}$. Then
\[ C_2 = 2 \left(\frac{\beta/(1-\beta)}{\sqrt{C_1(1-\gamma)\theta_1}}+1\right)^2. \]
Thus, we obtain Eq.~\eqref{2-012} by adding the term $L\mathbb{E}\left[\norm{\bm{\Delta}_t}^2\right]$ to both sides of Eq.~\eqref{1-030}.

Next, we estimate Eq.~\eqref{2-013}. When $t=1$, we have
\begin{equation}\label{2-031}
\begin{split}
M_1 =~& \mathbb{E}\left[-\left\langle \bm{\nabla} f(\bm{x}_1), \frac{\alpha_1\bm{m}_1}{\sqrt{\bm{v}_1}}\right\rangle + L \norm{\bm{\Delta}_1}^2\right]
= \mathbb{E}\left[-\left\langle \bm{\nabla} f(\bm{x}_1), \frac{\alpha_1(1-\beta_1)\bm{g}_1}{\sqrt{\bm{v}_1}}\right\rangle + L \norm{\bm{\Delta}_1}^2\right].
\end{split}
\end{equation}
The same as what we did for term (I) in Lemma \ref{lem1-002}, we have
\begin{equation}
\frac{(1-\beta_1)\alpha_1\bm{g}_1}{\sqrt{\bm{v}_t}} = (1-\beta_1)\bm{\hat{\eta}}_1\bm{g}_1 + \bm{\hat{\eta}}_1\bm{\sigma}_1\frac{(1-\theta_1)\bm{g}_1}{\sqrt{\bm{v}_1}}\frac{(1-\beta_1)\bm{\sigma}_1}{\sqrt{\bm{v}_1}+\sqrt{\bm{\hat{v}}_1}} -
\bm{\hat{\eta}}_1\bm{g}_1\frac{(1-\theta_1)\bm{g}_1}{\sqrt{\bm{v}_1}}\frac{(1-\beta_1)\bm{g}_1}{\sqrt{\bm{v}_1}+\sqrt{\bm{\hat{v}}}_1}.
\end{equation}
Then the similar argument as Eq.~\eqref{1-021} implies that
\begin{equation}\label{2-033}
\begin{split}
\mathbb{E}\left[-\left\langle\bm{\nabla} f(\bm{x}_1), \frac{\alpha_1\bm{m}_1}{\sqrt{\bm{v}_1}}\right\rangle\right]
\leq~& C_2G \chi_1\mathbb{E}\left[\norm{\frac{\sqrt{1-\theta_t}\bm{g}_1}{\sqrt{\bm{v}_1}}}^2\right] -\frac{1-\beta_1}{2}\mathbb{E}\left[\norm{\bm{\nabla} f(\bm{x}_1)}_{\bm{\hat{\eta}}_1}^2\right]\\
\leq~& C_2G \chi_1\mathbb{E}\left[\norm{\frac{\sqrt{1-\theta_t}\bm{g}_1}{\sqrt{\bm{v}}_1}}^2\right].
\end{split}
\end{equation}
Combining Eq.~\eqref{2-031} and Eq.~\eqref{2-033}, and adding both sides by $L\mathbb{E}\left[\norm{\bm{\Delta}}_1^2\right]$, we obtain Eq.~\eqref{2-013}. This completes the proof.
\end{proof}

\medskip

\begin{lemma}\label{lem1-005}
The following estimate holds
\begin{equation}
\sum_{t=1}^T \norm{\bm{\Delta}_t}^2 \leq \frac{C_0^2\chi_1}{C_1(1-\sqrt{\gamma})^2}\sum_{t=1}^T\chi_t \norm{\frac{\sqrt{1-\theta_t}\bm{g}_t}{\sqrt{\bm{v}_t}}}^2.
\end{equation}
\end{lemma}

\begin{proof}
Note that $\bm{v}_t \geq \theta_t \bm{v}_{t-1}$, so we have $\bm{v}_t \geq \left(\prod_{j=i+1}^t \theta_j\right) \bm{v}_{i} = \Theta_{(t,i)}\bm{v}_i$. By Lemma \ref{lem2-001}, this follows that $\bm{v}_t \geq C_1 (\theta')^{t-i}\bm{v}_{i}$ for all $i \leq t$. On the other hand,
\[ |\bm{m}_t| \leq \sum_{i=1}^t \left(\prod_{j=i+1}^t \beta_j\right)(1-\beta_i)|\bm{g}_i|
\leq \sum_{i=1}^t \beta^{t-i}|\bm{g}_i|. \]
It follows that
\begin{equation}
\begin{split}
\frac{|\bm{m}_t|}{\sqrt{\bm{v}_t}}
\leq \sum_{i=1}^t\frac{\beta^{t-i}|\bm{g}_i|}{\sqrt{\bm{v}_t}}
%
%
%
\leq \frac{1}{\sqrt{C_1}}\sum_{i=1}^t \left(\frac{\beta}{\sqrt{\theta'}}\right)^{t-i} \frac{|\bm{g}_i|}{\sqrt{\bm{v}_i}}
= \frac{1}{\sqrt{C_1}}\sum_{i=1}^t \sqrt{\gamma}^{t-i} \frac{|\bm{g}_i|}{\sqrt{\bm{v}_i}}.
\end{split}
\end{equation}
Since $\alpha_t = \chi_t\sqrt{1-\theta_t} \leq \chi_t \sqrt{1-\theta_i}$ for $i\leq t$, it follows that
\begin{equation}
\begin{split}
\norm{\bm{\Delta}_t}^2 = \norm{\frac{\alpha_t\bm{m}_t}{\sqrt{\bm{v}_t}}}^2
%
%
%
\leq~& \frac{\chi_t^2}{C_1}\norm{\sum_{i=1}^t \sqrt{\gamma}^{t-i}\frac{\sqrt{1-\theta_i}|\bm{g_i}|}{\sqrt{\bm{v}_i}}}^2
\leq \frac{\chi_t^2}{C_1}\left(\sum_{i=1}^t \sqrt{\gamma}^{t-i} \right)\sum_{i=1}^t \sqrt{\gamma}^{t-i}\norm{\frac{\sqrt{1-\theta_i}\bm{g}_i}{\sqrt{\bm{v}_i}}}^2 \\
\leq~& \frac{\chi_t^2}{C_1(1-\sqrt{\gamma})}\sum_{i=1}^t \sqrt{\gamma}^{t-i}\norm{\frac{\sqrt{1-\theta_i}\bm{g}_i}{\sqrt{\bm{v}_i}}}^2.
\end{split}
\end{equation}
By Lemma \ref{lem2-002},
\[ \chi_t \leq C_0\chi_i, \forall i \leq t. \]
Hence,
\begin{equation}
\begin{split}
\norm{\bm{\Delta}_t}^2 = \norm{\frac{\alpha_t\bm{m}_t}{\sqrt{\bm{v}_t}}}^2 \leq \frac{C_0^2\chi_1}{C_1(1-\sqrt{\gamma})}\sum_{i=1}^t \sqrt{\gamma}^{t-i}\chi_i\norm{\frac{\sqrt{1-\theta_i}\bm{g}_i}{\sqrt{\bm{v}_i}}}^2.
\end{split}
\end{equation}
It follows that
\begin{equation}
\begin{split}
\sum_{t=1}^T \norm{\bm{\Delta}_t}^2
\leq~& \frac{C_0^2\chi_1}{C_1(1-\sqrt{\gamma})}\sum_{t=1}^T\sum_{i=1}^t \sqrt{\gamma}^{t-i}\chi_i\norm{\frac{\sqrt{1-\theta_i}\bm{g}_i}{\sqrt{\bm{v}_i}}}^2 \\
=~& \frac{C_0^2\chi_1}{C_1(1-\sqrt{\gamma})}\sum_{i=1}^T\left(\sum_{t=i}^T \sqrt{\gamma}^{t-i}\right)\chi_i\norm{\frac{\sqrt{1-\theta_i}\bm{g}_i}{\sqrt{\bm{v}_i}}}^2 \\
\leq~& \frac{C_0^2\chi_1}{C_1(1-\sqrt{\gamma})^2}\sum_{i=1}^T\chi_i\norm{\frac{\sqrt{1-\theta_i}\bm{g}_i}{\sqrt{\bm{v}_i}}}^2.
\end{split}
\end{equation}
The proof is completed.
\end{proof}

\medskip

\begin{lemma}\label{lem1-006}
Let $M_t = \mathbb{E} \left[\langle \bm{\nabla} f(\bm{x}_{t}), \bm{\Delta}_{t}\rangle + L\norm{\bm{\Delta}_{t}}^2\right]$. For $T\geq 1$ we have
\begin{equation}\label{1-037}
\begin{split}
\sum_{t=1}^T M_t
%
\leq C_3\mathbb{E}\left[\sum_{t=1}^T\chi_t\norm{\frac{\sqrt{1-\theta_t}\bm{g}_t}{\sqrt{\bm{v}_t}}}^2\right]
- \frac{1-\beta}{2}\mathbb{E}\left[\sum_{t=1}^T \norm{\bm{\nabla} f(\bm{x}_t)}_{\bm{\hat{\eta}_t}}^2\right].
\end{split}
\end{equation}
where the constant $C_3$ is given by
\[ C_3= \frac{C_0}{\sqrt{C_1}(1-\sqrt{\gamma})}\left(\frac{C_0^2\chi_1L}{C_1(1-\sqrt{\gamma})^2} + 2\left(\frac{\beta/(1-\beta)}{\sqrt{C_1(1-\gamma)\theta_1}}+1\right)^2G\right).\]
\end{lemma}
\begin{proof}
Let $N_t = L\mathbb{E}\left[\norm{\bm{\Delta}_t}^2\right] + C_2G \chi_t\mathbb{E}\left[\norm{\frac{\sqrt{1-\theta_t}\bm{g}_t}{\sqrt{\bm{v}_t}}}^2\right]$. 
By Lemma \ref{lem1-004}, we have
$M_1 \leq N_1$ 
and
\begin{equation}
M_t \leq \frac{\beta_t\alpha_t}{\sqrt{\theta_t}\alpha_{t-1}}M_{t-1} + N_t - \frac{1-\beta}{2}\mathbb{E}\left[\norm{\bm{\nabla} f(\bm{x}_t)}_{\bm{\hat{\eta}}_t}^2\right] \leq \frac{\beta_t\alpha_t}{\sqrt{\theta_t}\alpha_{t-1}}M_{t-1} + N_t.
\end{equation}
It is straightforward to acquire by induction that
\begin{equation}
\begin{split}
M_t &\leq \frac{\beta_t\alpha_t}{\sqrt{\theta_t}\alpha_{t-1}}\frac{\beta_{t-1}\alpha_{t-1}}{\sqrt{\theta_{t-1}}\alpha_{t-2}}M_{t-2} + \frac{\beta_t\alpha_t}{\sqrt{\theta_t}\alpha_{t-1}}N_{t-1} + N_t - \frac{1-\beta}{2}\mathbb{E}\left[\norm{\bm{\nabla} f(\bm{x}_t)}_{\bm{\hat{\eta}}_t}^2\right]\\
&~ \vdots \\
&\leq \frac{\alpha_tB_{(t, 1)}}{\alpha_1\sqrt{\Theta_{(t,1)}}}M_1 + \sum_{i=2}^t \frac{\alpha_tB_{(t,i)}}{\alpha_i\sqrt{\Theta_{(t,i)}}}N_i - \frac{1-\beta}{2}\mathbb{E}\left[\norm{\bm{\nabla} f(\bm{x}_t)}_{\bm{\hat{\eta}}_t}^2\right]\\
&\leq \sum_{i=1}^t \frac{\alpha_tB_{(t,i)}}{\alpha_i\sqrt{\Theta_{(t,i)}}}N_i - \frac{1-\beta}{2}\mathbb{E}\left[\norm{\bm{\nabla} f(\bm{x}_t)}_{\bm{\hat{\eta}}_t}^2\right].
\end{split}
\end{equation}
By Lemma \ref{lem2-002}, it holds $\alpha_t \leq C_0\alpha_i$ for any $i \leq t$. By Lemma \ref{lem2-001}, $\Theta_{(t,i)} \geq C_1(\theta')^{t-i}$. In addition, $B_{(t,i)}\leq \beta^{t-i}$. Hence,
\begin{equation}
\begin{split}
M_t \leq~& \frac{C_0}{\sqrt{C_1}}\sum_{i=1}^t \left(\frac{\beta}{\sqrt{\theta'}}\right)^{t-i}N_i - \frac{1-\beta}{2}\mathbb{E}\left[\norm{\bm{\nabla} f(\bm{x}_t)}^2_{\bm{\hat{\eta}}_t}\right] \\
=~& \frac{C_0}{\sqrt{C_1}}\sum_{i=1}^t \sqrt{\gamma}^{t-i}N_i - \frac{1-\beta}{2}\mathbb{E}\left[\norm{\bm{\nabla} f(\bm{x}_t)}^2_{\bm{\hat{\eta}}_t}\right].
\end{split}
\end{equation}
Hence,
\begin{equation}\label{1-042}
\begin{split}
\sum_{t=1}^T M_t
\leq~& \frac{C_0}{\sqrt{C_1}} \sum_{t=1}^T\sum_{i=1}^t\sqrt{\gamma}^{t-i}N_i - \frac{1-\beta}{2}\mathbb{E}\left[\sum_{t=1}^T\norm{\bm{\nabla}f(\bm{x}_t)}^2_{\bm{\hat{\eta}}_t}\right] \\
=~& \frac{C_0}{\sqrt{C_1}}\sum_{i=1}^T \left(\sum_{t=i}^T\sqrt{\gamma}^{t-i}\right)N_i - \frac{1-\beta}{2}\mathbb{E}\left[\sum_{t=1}^T\norm{\bm{\nabla}f(\bm{x}_t)}^2_{\bm{\hat{\eta}}_t}\right]\\
=~& \frac{C_0}{\sqrt{C_1}(1-\sqrt{\gamma})}\sum_{t=1}^T N_t - \frac{1-\beta}{2}\mathbb{E}\left[\sum_{t=1}^T\norm{\bm{\nabla}f(\bm{x}_t)}^2_{\bm{\hat{\eta}}_t}\right].
\end{split}
\end{equation}
Finally, by Lemma \ref{lem1-005}, we have
\begin{equation}\label{1-043}
\begin{split}
\sum_{t=1}^T N_i =~& \mathbb{E}\left[L\sum_{t=1}^T\norm{\bm{\Delta}_t}^2 + C_2G\sum_{t=1}^T\chi_t\norm{\frac{\sqrt{1-\theta_t}\bm{g}_t}{\sqrt{\bm{v}_t}}}^2\right]\\
\leq~& \left(\frac{C_0^2\chi_1L}{C_1(1-\sqrt{\gamma})^2}+C_2G\right)\mathbb{E}\left[\sum_{t=1}^T\chi_t\norm{\frac{\sqrt{1-\theta_t}\bm{g}_t}{\sqrt{\bm{v}_t}}}^2\right].
\end{split}
\end{equation}
Let
\[\begin{split}
C_3 =~& \frac{C_0}{\sqrt{C_1}(1-\sqrt{\gamma})}\left(\frac{C_0^2\chi_1L}{C_1(1-\sqrt{\gamma})^2} + C_2 G\right)\\
=~& \frac{C_0}{\sqrt{C_1}(1-\sqrt{\gamma})}\left(\frac{C_0^2\chi_1L}{C_1(1-\sqrt{\gamma})^2} + 2\left(\frac{\beta/(1-\beta)}{\sqrt{C_1(1-\gamma)\theta_1}}+1\right)^2G\right).
\end{split}\]
Combining Eq.~\eqref{1-042} and Eq.~\eqref{1-043}, we then obtain the desired estimate Eq.~\eqref{1-037}. The proof is completed.
\end{proof}

\medskip

\begin{lemma}\label{lem1-007}
The following estimate holds
\end{lemma}
\begin{equation}\label{1-044}
\mathbb{E}\left[\sum_{i=1}^t\norm{\frac{\sqrt{1-\theta_i}\bm{g}_i}{\sqrt{\bm{v}_i}}}^2\right]
\leq d\left[\log\left(1 + \frac{G^2}{\epsilon d}\right) + \sum_{i=1}^t\log(\theta_i^{-1})\right].
\end{equation}
\begin{proof}
Let $W_0 = 1$ and $W_t = \prod_{i=1}^T \theta_i^{-1}$. Let $w_t = W_t - W_{t-1} = (1 - \theta_t)\prod_{i=1}^{t}\theta_i^{-1} = (1-\theta_t)W_t$. We therefore have
\[ \frac{w_t}{W_t} = 1-\theta_t, \quad \frac{W_{t-1}}{W_t} = \theta_t.\]
Note that $\bm{v}_0 = \bm{\epsilon}$ and $\bm{v}_t = \theta_t \bm{v}_{t-1} + (1-\theta_t)\bm{g}_t$, so it holds that $W_0\bm{v}_0 = \bm{\epsilon}$ and
$
W_t\bm{v}_t = W_{t-1}\bm{v}_{t-1} + w_t\bm{g}_t^2.
$
Then,
$
W_t\bm{v}_t = W_0\bm{v}_0 + \sum_{i=1}^t w_i\bm{g}_i^2 = \bm{\epsilon} + \sum_{i=1}^t w_i\bm{g}_i^2.
$
It follows that
\begin{equation}\label{1-049}
\begin{split}
\sum_{i=1}^t\norm{\frac{\sqrt{1-\theta_i}\bm{g}_i}{\sqrt{\bm{v}_i}}}^2
=~& \sum_{i=1}^t \norm{\frac{(1-\theta_i)\bm{g}_t^2}{\bm{v}_i}}_1
= \sum_{i=1}^t \norm{\frac{w_i\bm{g}_i^2}{W_i\bm{v}_i}}_1
= \sum_{i=1}^t \norm{\frac{w_i\bm{g}_i^2}{\bm{\epsilon} + \sum_{\ell=1}^i w_\ell\bm{g}_\ell^2}}_1.
\end{split}
\end{equation}
Writing the norm in terms of coordinates, we obtain
\begin{equation}
\sum_{i=1}^t \norm{\frac{\sqrt{1-\theta_i}\bm{g}_i}{\sqrt{\bm{v}_i}}}^2
= \sum_{i=1}^t \sum_{k=1}^d \frac{w_i g_{i,k}^2}{ {\epsilon} + \sum_{\ell=1}^i w_\ell g_{\ell,k}^2}
= \sum_{k=1}^d\sum_{i=1}^t \frac{w_i g_{i,k}^2}{ {\epsilon} + \sum_{\ell=1}^i w_\ell g_{\ell,k}^2}.
\end{equation}
By Lemma \ref{lem2-001}, for each $k = 1,2,\ldots,d$,
\begin{equation}
\sum_{i=1}^t \frac{w_i g_{i,k}^2}{\epsilon + \sum_{\ell=1}^i w_\ell g_{\ell,k}^2} \leq \log\left(\epsilon + \sum_{\ell=1}^t w_\ell g_{\ell,k}^2\right) - \log(\epsilon) = \log\left(1 + \frac{1}{\epsilon}\sum_{\ell=1}^t w_\ell g_{\ell,k}^2\right).
\end{equation}
Hence,
\begin{equation}\label{2-050}
\begin{split}
\sum_{i=1}^t\norm{\frac{\sqrt{1-\theta_i}\bm{g}_i}{\sqrt{\bm{v}_i}}}^2
\leq~& \sum_{k=1}^d \log\left(1 + \frac{1}{\epsilon}\sum_{i=1}^t w_i g_{i,k}^2\right) \\
\leq~& d\log\left(\frac{1}{d}\sum_{k=1}^d\left(1 + \frac{1}{\epsilon}\sum_{i=1}^t w_i  g_{i,k}^2\right)\right)
= d\log\left(1 + \frac{1}{\epsilon d}\sum_{i=1}^t w_i\norm{\bm{g}_i}^2\right).
\end{split}
\end{equation}
The second inequality is due to the convex inequality $\frac{1}{d}\sum_{k=1}^d\log\left(z_i\right) \leq \log\left(\frac{1}{d}\sum_{k=1}^d z_i\right)$. Indeed, we have the more general convex inequality that
\begin{equation}
\mathbb{E}[\log(X)] \leq \log{\mathbb{E}[X]},
\end{equation}
for any positive random variable $X$. Taking $X$ to be $1 + \frac{1}{\epsilon d}\sum_{i=1}^t w_i \norm{\bm{g}_i}^2$ in the right hand side of Eq.~\eqref{2-050}, we obtain that
\begin{equation}
\begin{split}
&\mathbb{E}\left[\sum_{i=1}^t\norm{\frac{\sqrt{1-\theta_i}\bm{g}_i}{\sqrt{\bm{v}_i}}}^2\right]
\leq d\ \mathbb{E}\left[\log\left(1 + \frac{1}{\epsilon d}\sum_{i=1}^t w_i \norm{\bm{g}_i}^2\right)\right]
\leq d\log\left(1 + \frac{1}{\epsilon d}\sum_{i=1}^t w_i \mathbb{E}\left[\norm{\bm{g}_i}^2\right]\right)\\
\leq~& d\log\left(1 + \frac{G^2}{\epsilon d}\sum_{i=1}^t w_i\right)
= d\log\left(1 + \frac{G^2}{\epsilon d}(W_t-W_0)\right)
= d\log\left(1 + \frac{G^2}{\epsilon d}\left(\prod_{i=1}^t\theta_i^{-1} -1\right)\right)\\
\leq~& d\left[\log\left(1+\frac{G^2}{\epsilon d}\right) + \log\left(\prod_{i=1}^t\theta_i^{-1}\right)\right].
\end{split}
\end{equation}
The last inequality is due to the following trivial inequality:
\[\log(1+ ab) \leq \log(1+a+b+ab) = \log(1+a) + \log(1+b) \]
for any non-negative parameters $a$ and $b$. It then follows that
\begin{equation}
\mathbb{E}\left[\sum_{i=1}^t\norm{\frac{\sqrt{1-\theta_i}\bm{g}_i}{\sqrt{\bm{v}_i}}}^2\right] \leq d\left[\log\left(1+\frac{G^2}{\epsilon d}\right) + \sum_{i=1}^t\log(\theta_i^{-1})\right].
\end{equation}
The proof is completed.
\end{proof}

\medskip

\begin{lemma}
We have the following estimate
\begin{equation}
\mathbb{E}\left[\sum_{t=1}^T \chi_t \norm{\frac{\sqrt{1-\theta_t}\bm{g}_t}{\sqrt{\bm{v}}_t}}^2\right] \leq C_0d\left[\chi_1\log\left(1 + \frac{G^2}{\epsilon d}\right) + \frac{1}{\theta_1}\sum_{t=1}^T \alpha_t\sqrt{1-\theta_t}\right].
\end{equation}
\end{lemma}

\begin{proof}
For simplicity of notations, let $\omega_t := \norm{\frac{\sqrt{1-\theta_t}\bm{g}_t}{\sqrt{\bm{v}_t}}}^2$, and $\Omega_t := \sum_{i=1}^t \omega_i$. Note that $\chi_t \leq C_0 a_t$. Hence,
\begin{equation}\label{3-059}
\mathbb{E}\left[\sum_{t=1}^T \chi_t \norm{\frac{\sqrt{1-\theta_t}\bm{g}_t}{\sqrt{\bm{v}}_t}}^2\right]
\leq C_0\ \mathbb{E}\left[\sum_{t=1}^T a_t \omega_t\right].
\end{equation}
By Lemma \ref{lem2-003}, we have
\begin{equation}\label{3-060}
\mathbb{E}\left[\sum_{t=1}^T a_t\omega_t\right] = \mathbb{E}\left[\sum_{t=1}^{T-1} (a_t - a_{t+1}) \Omega_t + a_T\Omega_T\right].
\end{equation}
Let $S_t := \log\left(1 + \frac{G^2}{\epsilon d}\right) + \sum_{i=1}^t \log(\theta_i^{-1})$. By Lemma \ref{lem1-007}, we have
\begin{equation}
\mathbb{E}[\Omega_t] \leq d S_t.
\end{equation}
Since $\{a_t\}$ is a non-increasing sequence, we have $a_t - a_{t+1} \geq 0$.  By Eq.~\eqref{3-060}, we have
\begin{equation}\label{3-061}
\begin{split}
& \mathbb{E}\left[\sum_{t=1}^{T-1} (a_t - a_{t+1}) \Omega_t + a_T\Omega_T\right]
\leq d\left(\sum_{t=1}^{T-1}(a_t - a_{t+1})S_t + a_T S_T\right) \\
=& d\left(a_1S_0 + \sum_{t=1}^T a_t(S_t - S_{t-1})\right)
= d\left[a_1\log\left(1 + \frac{G^2}{\epsilon d}\right) + \sum_{t=1}^T a_t\log(\theta_t^{-1})\right].
\end{split}
\end{equation}
Note that $a_t \leq \chi_t$. Combining Eq.~\eqref{3-059}, Eq.~\eqref{3-060}, and Eq.~\eqref{3-061}, we have
\begin{equation}
\begin{split}
\mathbb{E}\left[\sum_{t=1}^T \chi_t \norm{\frac{\sqrt{1-\theta_t}\bm{g}_t}{\sqrt{\bm{v}}_t}}^2\right]
\leq~& C_0 d \left[\chi_1\log\left(1 + \frac{G^2}{\epsilon d}\right) + \sum_{t=1}^T \chi_t\log(\theta_t^{-1})\right] \\
=~& C_0 d \left[ \chi_1\log\left(1 + \frac{G^2}{\epsilon d}\right) + \sum_{t=1}^T \chi_t\log(\theta_t^{-1})\right].
\end{split}
\end{equation}
Note that $\log(1+x) \leq x$ for all $x > -1$. It follows that
$$\log(\theta_t^{-1}) = \log(1 + (\theta_t^{-1} - 1)) \leq \theta_t^{-1} - 1 \leq \frac{1-\theta_t}{\theta_1}.$$
Note that $\chi_t = \alpha_t/\sqrt{1-\theta_t}$. By Eq.~\eqref{3-059} and Eq.~\eqref{3-061}, we have
\begin{equation}
\mathbb{E}\left[\sum_{t=1}^T \chi_t \norm{\frac{\sqrt{1-\theta_t}\bm{g}_t}{\sqrt{\bm{v}}_t}}^2\right]
\leq C_0d\left[\chi_1\log\left(1+\frac{G^2}{\epsilon d}\right) - \frac{1}{\theta_1}\sum_{t=1}^T \alpha_t \sqrt{1-\theta_t}\right].
\end{equation}
The proof is completed.
\end{proof}

\medskip

\begin{lemma}\label{lem1-008}
Let $\tau$ be randomly chosen from $\{1, 2, \ldots, T\}$ with equal probabilities $p_\tau = 1/T$. We have the following estimate
\begin{equation}
\left(\mathbb{E}\left[\norm{\bm{\nabla} f(\bm{x}_\tau)}^{4/3}\right]\right)^{3/2} \leq \frac{C_0\sqrt{G^2 +\epsilon d}}{T\alpha_T}\ \mathbb{E}\left[\sum_{t=1}^T\norm{\bm{\nabla} f(\bm{x}_t)}_{\bm{\hat{\eta}}_t}^2\right].
\end{equation}
\end{lemma}
\begin{proof}
For any two random variables $X$ and $Y$, by the H\"{o}lder's inequality, we have
\begin{equation}\label{1-053}
\mathbb{E}[|XY|] \leq \mathbb{E}\left[|X|^p\right]^{1/p} \mathbb{E}\left[|Y|^q\right]^{1/q}.
\end{equation}
Let $X = \left(\frac{\norm{\bm{\nabla} f(\bm{x}_t)}^2}{\sqrt{\norm{\bm{\hat{v}}_t}_1}}\right)^{2/3}$, $Y = \norm{\bm{\hat{v}}_t}_1^{1/3}$, and let $p = 3/2$, $q = 3$. By Eq.~\eqref{1-053}, we have
\begin{equation}\label{1-054}
\mathbb{E}\left[\norm{\bm{\nabla} f(\bm{x}_t)}^{4/3}\right] \leq \mathbb{E}\left[\frac{\norm{\bm{\nabla} f(\bm{x}_t)}^2}{\sqrt{\norm{\bm{\hat{v}}_t}_1}}\right]^{2/3} \mathbb{E}\left[\norm{\bm{\hat{v}}_t}_1 \right]^{1/3}.
\end{equation}
On the one hand, we have
\begin{equation}\label{1-055}
\begin{split}
\frac{\norm{\bm{\nabla} f(\bm{x}_t)}^2}{\sqrt{\norm{\bm{\hat{v}}_t}_1}}
= \sum_{k=1}^d \frac{|\nabla_k f(\bm{x}_t)|^2}{\sqrt{\sum_{k=1}^d \hat{v}_{t,k}} }
\leq~& \alpha_t^{-1} \sum_{k=1}^d \frac{\alpha_t}{\sqrt{\hat{v}_{t,k}}}|\nabla_k f(\bm{x}_t)|^2 \\
=~& \alpha_t^{-1} \sum_{k=1}^d \hat{\eta}_{t,k}|\nabla_k f(\bm{x}_t)|^2
= \alpha_t^{-1}\norm{\bm{\nabla} f(\bm{x}_t)}_{\bm{\hat{\eta}_t}}^2.
\end{split}
\end{equation}
Since $\bm{\hat{v}}_t = \theta_t \bm{{v}}_{t-1} + (1-\theta_t) \bm{\sigma}_t^2$, and all entries are non-negative, we have \[\norm{\bm{\hat{v}}_t}_1 = \theta_t \norm{\bm{v}_{t-1}}_1 + (1-\theta_t)\norm{\bm{\sigma}_t}^2.\]
Notice that $\bm{v}_t = \theta_t \bm{v}_{t-1} + (1 - \theta_t)\bm{g}_t^2$, $\bm{v}_0 = \bm{\epsilon}$, and $\mathbb{E}_t\left[\bm{g}_t^2\right] \leq G^2$. It is straightforward to prove by induction that $\mathbb{E}[\norm{\bm{v}_t}_1] \leq G^2 + \epsilon d$. Hence,
\begin{equation}\label{1-056}
\mathbb{E}[\norm{\bm{\hat{v}}_t}_1] \leq G^2 + \epsilon d.
\end{equation}
By Eq.~\eqref{1-054}, Eq.~\eqref{1-055}, and Eq.~\eqref{1-056}, we obtain
\begin{equation}
\mathbb{E}\left[\norm{\bm{\nabla} f(\bm{x}_t)}^{4/3}\right] \leq \left(\alpha_t^{-1}\mathbb{E}\left[\norm{\bm{\nabla} f(\bm{x}_t)}_{\bm{\hat{\eta}_t}}^2\right]\right)^{2/3} (G^2 + \epsilon d)^{1/3}.
\end{equation}
By Lemma \ref{lem2-002}, $\alpha_T \leq C_0 \alpha_t$ for any $t\leq T$, so $\alpha_t^{-1} \leq C_0\alpha_T^{-1}$. Then, we obtain
\begin{equation}
\mathbb{E}\left[\norm{\bm{\nabla} f(\bm{x}_t)}^{4/3}\right]^{3/2} \leq \frac{C_0\sqrt{G^2+\epsilon d}}{\alpha_T}\mathbb{E}\left[\norm{\bm{\nabla} f(\bm{x}_t)}_{\bm{\hat{\eta}_t}}^2\right],~\forall t\leq T.
\end{equation}
The lemma is followed by
\begin{equation}\begin{split}
\left(\mathbb{E}\left[\norm{\bm{\nabla} f(\bm{x}_\tau)}^{4/3}\right]\right)^{3/2} =~& \left(\frac{1}{T}\sum_{t=1}^T \mathbb{E}\left[\norm{\bm{\nabla} f(\bm{x}_t)}^{4/3}\right]\right)^{3/2} \\
\leq~& \frac{1}{T}\sum_{t=1}^T\left(\mathbb{E}\left[\norm{\bm{\nabla} f(\bm{x}_t)}^{4/3}\right]\right)^{3/2}
\leq \frac{C_0\sqrt{G^2+\epsilon d}}{T\alpha_T}\ \mathbb{E}\left[\sum_{t=1}^T \norm{\bm{\nabla} f(\bm{x}_t)}_{\bm{\hat{\eta}_t}}^2\right].
\end{split}\end{equation}
The proof is completed.
\end{proof}

\medskip

\section{Proofs of the main results}\label{main-proof-all}
In this section, we provide the detailed proofs of the propositions, theorems, and corollaries in the main body.
\subsection{Proof of Proposition \ref{equivalence-theorem}}\label{main-proof-prop}
\begin{proposition*}
Algorithm \ref{Adam} and Algorithm \ref{Weithed AdamEMA} are equivalent.
\end{proposition*}
\begin{proof}
It suffices to show that Algorithm \ref{Adam} can be realized as Algorithm \ref{Weithed AdamEMA} with a particular choice of parameters, and vice versa. Note that for Algorithm \ref{Adam}, it holds
\begin{equation}\label{2-008}
\bm{x}_{t+1} \!=\! \bm{x}_t \!-\! \frac{\alpha_t \bm{m}_t}{\sqrt{\big(\prod\limits_{i=1}^t \theta_i\big)\bm{\epsilon} + \sum\limits_{i=1}^t\!\big(\!\prod\limits_{j=i+1}^t \theta_j (1-\theta_i)\big)\bm{g}_i^2}}.
\end{equation}
While for Algorithm \ref{Weithed AdamEMA}, we have
\begin{equation}\label{2-009}
\bm{x}_{t+1} = \bm{x}_t - \frac{\alpha_t \bm{m}_t}{\sqrt{\frac{1}{W_t}\bm{\epsilon} + \sum_{i=1}^t \frac{w_i}{W_t} \bm{g}_t^2}}.
\end{equation}
Hence, given the parameters $\theta_t$ in Algorithm \ref{Adam}, we take $w_t = (1-\theta_t)\prod_{i=1}^t \theta_i^{-1}$. Then it holds
\[
W_t = 1 + \sum_{i=1}^t w_i = \prod_{i=1}^t \theta_i^{-1}.
\]
It follows that Eq.~\eqref{2-009} becomes Eq.~\eqref{2-008}. Conversely, given the parameters $w_t$ of Algorithm \ref{Weithed AdamEMA}, we take
$\theta_t = W_{t-1}/W_t$. Then Eq.~\eqref{2-008} becomes Eq.~\eqref{2-009}. The proof is completed.
\end{proof}

\subsection{Proof of Theorem \ref{convergence_in_expectation} }
\begin{theorem*}
Let $\{\bm{x}_t\}$ be a sequence generated by Generic Adam for initial values $\bm{x}_1$, $\bm{m}_0 =\bm{0}$, and $\bm{v}_0 =\bm{\epsilon}$. Assume that $f$ and stochastic gradients $\bm{g}_t$ satisfy assumptions (A1)-(A4). Let $\tau$ be randomly chosen from $\{1, 2, \ldots, T\}$ with equal probabilities $p_\tau = 1/T$. We have the following estimate
\begin{equation}\label{1-058}
\left(\mathbb{E}\left[\norm{\bm{\nabla} f(\bm{x}_\tau)}^{4/3} \right]\right)^{3/2}
\leq \frac{C + C'\sum_{t=1}^T\alpha_t\sqrt{1-\theta_t}}{T\alpha_T},
\end{equation}
where the constants $C$ and $C'$ are given by
\begin{equation*}
\begin{split}
C &= \frac{2C_0\sqrt{G^2+\epsilon d}}{1-\beta}\left(f(x_1) - f^* + C_3C_0d\ \chi_1\log\left(1+ \frac{G^2}{\epsilon d}\right)\right), \\
C' &= \frac{2C_0^2C_3d\sqrt{G^2 + \epsilon d}}{(1-\beta)\theta_1}.
\end{split}
\end{equation*}
\end{theorem*}
\begin{proof}
By the $L$-Lipschitz continuity of the gradient of $f$ and the descent lemma, we have
\begin{equation}
f(\bm{x}_{t+1}) \leq f(\bm{x}_t) + \langle \bm{\nabla} f(\bm{x}_t), \bm{\Delta}_t \rangle + \frac{L}{2} \norm{\bm{\Delta}_t}^2.
\end{equation}
Let $M_t := \mathbb{E}\left[\langle \bm{\nabla} f(\bm{x}_t), \bm{\Delta}_t \rangle + L \norm{\bm{\Delta}_t}^2\right]$.
We have $\mathbb{E}[f(\bm{x}_{t+1})] \leq \mathbb{E}[f(\bm{x}_t)] + M_t$.
Taking a sum for $t=1,2,\ldots,T$, we obtain
\begin{equation}
\mathbb{E}\left[f(\bm{x}_{T+1})\right] \leq f(\bm{x}_1) + \sum_{t=1}^T M_t.
\end{equation}
Note that $f(x)$ is bounded from below by $f^*$, so $\mathbb{E}[f(\bm{x}_{T+1})] \geq f^*$. Applying the estimate of Lemma \ref{lem1-006}, we have
\begin{equation}\label{1-062}
f^* \leq f(\bm{x}_1) + C_3\mathbb{E}\left[\sum_{t=1}^T\chi_t\norm{\frac{\sqrt{1-\theta_t}\bm{g}_t}{\sqrt{\bm{v}_t}}}^2\right]
- \frac{1-\beta}{2}\mathbb{E}\left[\sum_{t=1}^T \norm{\bm{\nabla} f(\bm{x}_t)}_{\bm{\hat{\eta}_t}}^2\right],
\end{equation}
where $C_3$ is the constant given in Lemma \ref{lem1-006}.
By applying the estimates in Lemma \ref{lem1-007} and Lemma \ref{lem1-008} for the second and third terms in the right hand side of Eq.~\eqref{1-062}, and appropriately rearranging the terms, we obtain
\begin{equation}
\begin{split}
&\left(\mathbb{E}\left[\norm{\bm{\nabla} f(\bm{x}^{T}_\tau)}^{4/3}\right]\right)^{3/2}
\leq \frac{C_0\sqrt{G^2+\epsilon d}}{T\alpha_T}\mathbb{E}\left[\sum_{t=1}^T\norm{\bm{\nabla} f(\bm{x}_t)}_{\bm{\hat{\eta}}_t}^2\right]\\
\leq~& \frac{2C_0\sqrt{G^2+\epsilon d}}{(1-\beta)T\alpha_T}\left(f(\bm{x}_1) - f^* +
C_3 \mathbb{E}\left[\sum_{t=1}^T\chi_t\norm{\frac{\sqrt{1-\theta_t}\bm{g}_t}{\sqrt{\bm{v}_t}}}\right]\right)\\
\leq~& \frac{2C_0\sqrt{G^2+\epsilon d}}{(1-\beta)T\alpha_T}\left[f(\bm{x}_1) - f^* + C_3C_0d\ \chi_1\log\left(1+\frac{G^2}{\epsilon d}\right) - \frac{C_3C_0d}{\theta_1} \sum_{t=1}^T\alpha_t\sqrt{1-\theta_t}\right] \\
=~& \frac{C + C'\sum_{t=1}^T \alpha_t\sqrt{1-\theta_t}}{T\alpha_T},
\end{split}
\end{equation}
where
\begin{equation*}
\begin{split}
C &= \frac{2C_0\sqrt{G^2+\epsilon d}}{1-\beta}\left(f(x_1) - f^* + C_3C_0d\ \chi_1\log\left(1+ \frac{G^2}{\epsilon d}\right)\right), \\
C' &= \frac{2C_0^2C_3d\sqrt{G^2 + \epsilon d}}{(1-\beta)\theta_1}.
\end{split}
\end{equation*}
The proof is completed.
\end{proof}

\subsection{Proof of Theorem \ref{thm1-001}}\label{main-proof}
\begin{theorem*}
Let $\{\bm{x}_t\}$ be a sequence generated by Generic Adam for initial values $\bm{x}_1$, $\bm{m}_0 =\bm{0}$, and $\bm{v}_0 =\bm{\epsilon}$. Assume that $f$ and stochastic gradients $\bm{g}_t$ satisfy assumptions (A1)-(A4). Let $\tau$ be randomly chosen from $\{1, 2, \ldots, T\}$ with equal probabilities $p_\tau = 1/T$. Then for any $\delta > 0$, the following bound holds with probability at least $1-\delta^{2/3}$:
\begin{equation}
\begin{split}
\norm{\bm{\nabla} f(\bm{x}_\tau)}^2
\leq 
\frac{C + C' \sum_{t=1}^T \alpha_t\sqrt{1-\theta_t}}{\delta T\alpha_T} := Bound(T),
\end{split}
\end{equation}
where the constants $C$ and $C'$ are given by
\begin{equation*}
\begin{split}
C &= \frac{2C_0\sqrt{G^2+\epsilon d}}{1-\beta}\left(f(x_1) - f^* + C_3C_0d\ \chi_1\log\left(1+ \frac{G^2}{\epsilon d}\right)\right), \\
C' &= \frac{2C_0^2C_3d\sqrt{G^2 + \epsilon d}}{(1-\beta)\theta_1},
\end{split}
\end{equation*}
in which the constant $C_3$ is given by
\begin{equation*}
C_3 = \frac{C_0}{\sqrt{C_1}(1-\sqrt{\gamma})}\left(\frac{C_0^2\chi_1L}{C_1(1-\sqrt{\gamma})^2} + 2\left(\frac{\beta/(1-\beta)}{\sqrt{C_1(1-\gamma)\theta_1}}+1\right)^2G\right).
\end{equation*}
\end{theorem*}
\begin{proof}
Denote the right hand side of Eq.~\eqref{1-058} as $C(T)$. Let $\zeta = \norm{\nabla f(x_\tau)}^2$. By Theorem \ref{convergence_in_expectation} we have $\mathbb{E}\left[|\zeta|^{2/3}\right] \leq C(T)^{2/3}$. Let $\mathcal{P}$ denote the probability measure. By Chebyshev's inequality, we have
\begin{equation}
    \mathcal{P}\left(|\zeta|^{2/3} > \frac{C(T)^{2/3}}{\delta^{2/3}} \right) \leq \frac{\mathbb{E}\left[|\zeta|^{2/3}\right]}{\frac{C(T)^{2/3}}{\delta^{2/3}}} \leq \delta^{2/3}.
\end{equation}
Namely, $\mathcal{P}\left(|\zeta| > \frac{C(T)}{\delta }\right) \leq \delta^{2/3}$. Therefore, $\mathcal{P}\left(|\zeta| \leq \frac{C(T)}{\delta }\right) \geq 1-\delta^{2/3}$. This completes the proof.
\end{proof}

\subsection{Proof of Corollary \ref{constant-theta}}
\begin{corollary*}
Take $\alpha_t = \eta/t^s$ with $0\leq s < 1$. Suppose $\lim_{t\to\infty} \theta_t = \theta < 1$.Then $Bound(T)$ in Theorem \ref{thm1-001} is bounded from below by constants
\begin{equation}
Bound(T) \geq \frac{C'\sqrt{1-\theta}}{\delta}.
\end{equation}
In particular, when $\theta_t = \theta < 1$, we have the following more subtle estimate on lower and upper-bounds for $Bound(T)$:
\begin{equation*}
\frac{C}{\delta \eta T^{1-s}} +\frac{C'\sqrt{1-\theta}}{\delta}\leq Bound(T) \!\leq\! \frac{C}{\delta \eta T^{1-s}} \!+\! \frac{C'\sqrt{1\!-\!\theta}}{\delta(1-s)}.
\end{equation*}
\end{corollary*}
\begin{proof}
Since $\lim_{t\to\infty} \theta_t = \theta$, and $\theta_t$ is non-decreasing, we have $(1-\theta_t) \geq 1-\theta$. Hence, by Theorem \ref{thm1-001}, it holds
\begin{align}\label{3-010}
Bound(T)
& \geq \frac{C}{\delta\eta T^{1-s}} + \frac{C'\sqrt{1-\theta}}{\delta}\big(\frac{\sum_{t=1}^Tt^{-s}}{T^{1-s}}\big) \nonumber\\
& \geq \frac{C'\sqrt{1-\theta}}{\delta}.
\end{align}
If, in particular, $\theta_t = \theta < 1$, by Theorem \ref{thm1-001} we have
\begin{equation}\label{3-011}
Bound(T) = \frac{C}{\delta\eta T^{1-s}} + \frac{C'\sqrt{1-\theta}}{\delta}\big(\frac{\sum_{t=1}^Tt^{-s}}{T^{1-s}}\big).\end{equation}
Note that
\begin{equation} \label{3-012}
1\!\leq\! \frac{{\textstyle\sum_{t=1}^T}t^{-s}}{T^{1-s}}\!=\! {\textstyle\sum\limits_{t=1}^T} \big(\frac{t}{T}\big)^{-s}\frac{1}{T} \!\leq\! \int_0^1 x^{-s}dx \!=\! \frac{1}{1\!-\!s}.
\end{equation}
Combining Eqs.~\eqref{3-011}-\eqref{3-012}, we obtain the desired result.
\end{proof}

\subsection{Proof of Corollary \ref{poly-setting}}\label{proof-last}
\begin{corollary*}
Generic Adam with the above family of parameters converges as long as $0 < r \leq 2s < 2$, and its non-asymptotic convergence rate is given by
\begin{equation*}
\norm{\bm{\nabla} f(\bm{x}_\tau)}^2 \leq \left\{\begin{aligned}
& \mathcal{O}(T^{-r/2}), \quad &   r/2 + s < 1 \\
& \mathcal{O}(\log(T)/T^{1-s}), \quad &  r/2 + s = 1 \\
& \mathcal{O}(1/T^{1 - s}), \quad &  r/2 + s > 1
\end{aligned}\right..
\end{equation*}
\end{corollary*}
\begin{proof}
It is not hard to verify that the following equalities hold:
\begin{align*}
\textstyle\sum_{t=K}^T \alpha_t\sqrt{1-\theta_t}
&= \eta\sqrt{\alpha}\textstyle\sum_{t=K}^T t^{-(r/2 + s)} \nonumber\\
&=\left\{
\begin{aligned}
& \mathcal{O}(T^{1-(r/2 +s)}),  & r/2 + s < 1 \\
& \mathcal{O}(\log(T)),  &  r/2 + s = 1 \\
& \mathcal{O}(1),  & r/2 + s > 1
\end{aligned}
\right..
\end{align*}
In this case, $T\alpha_T = \eta T^{1-s}$.
Therefore, by Theorem \ref{thm1-001} the non-asymptotic convergence rate is given by
\[
\norm{\bm{\nabla} f(\bm{x}_\tau)}^2 \leq \left\{
\begin{aligned}
& \mathcal{O}(T^{-r/2}),  &  r/2 + s < 1 \\
& \mathcal{O}(\log(T)/T^{1-s}),  & r/2 + s = 1 \\
& \mathcal{O}(1/T^{1-s}),  & r/2 + s > 1
\end{aligned}
\right..\]
To guarantee convergence, then $0 < r \leq 2s < 2$.
\end{proof}

\subsection{Proof of Corollary \ref{polyweights}}
\begin{corollary*}
Suppose that in Weighted AdaEMA the weights $w_t = t^r$ for $r\!\geq\! 0$, and $\alpha_t \!=\!\eta/\sqrt{t}$. Then Weighted AdaEMA has the $\mathcal{O}(\log(T)/\sqrt{T})$ non-asymptotic convergence rate.
\end{corollary*}
\begin{proof}
By the proof procedures of Theorem \ref{equivalence-theorem}, the equivalent Generic Adam has the parameters $\theta_t = W_{t-1}/W_t$, where $W_t = 1 + \sum_{i=1}^t w_i$.
Hence, it holds
\[ 1 - \theta_t = \frac{w_t}{W_t} = \frac{t^{r}}{1 + \sum_{i=1}^t i^r} = \mathcal{O}(1/t). \]
We have that $\lim_{t\to\infty} \theta_t = 1 > \beta$ and $\theta_t$ is increasing. In addition, we have that $\chi_t = \alpha_t/\sqrt{1-\theta_t}$ is bounded, and hence ``almost" non-increasing (by taking $a_t = 1$ in (R3)). The restrictions (R1)-(R3) are all satisfied. Hence, we can apply Theorem \ref{thm1-001} in this case. It follows that its convergence rate is given by
\[
\mathcal{O}\big(\frac{\sum_{i=1}^T \alpha_t\sqrt{1-\theta_t}}{{T}\alpha_T}\big)
= \mathcal{O} \big(\frac{\sum_{t=1}^T 1/t}{\sqrt{T}}\big)
= \mathcal{O}\big(\frac{\log(T)}{\sqrt{T}}\big).
\]
The proof is completed.
\end{proof}

\section{Experimental Implementations}\label{LeNet-and-ResNet}
In this section, we describe the statistics of the training and validation datasets of MNIST\footnote{http://yann.lecun.com/exdb/mnist/} and CIFAR-100\footnote{https://www.cs.toronto.edu/~kriz/cifar.html}, the architectures of LeNet and ResNet-18, and detailed implementations.

\subsection{Datasets}
MNIST \cite{lecun2010mnist} is composed of ten classes of digits among $\{0, 1, 2, \ldots, 9\}$, which includes 60,000 training examples and 10,000 validation examples. The dimension of each example is $28 \times 28$.

CIFAR-100 \cite{lecun2010mnist} is composed of 100 classes of $32\times 32$ color images. Each class includes 6,000 images.
In addition, these images are devided into 50,000 training examples and 10,000 validation examples.

\subsection{Architectures of Neural Networks}
LetNet \cite{lecun1998gradient} used in the experiments is a five-layer convolutional neural network with ReLU activation function whose detailed architecture is described in~\cite{lecun1998gradient}. The batch size is set as $64$. The training stage lasts for $100$ epochs in total. No $\ell_2$ regularization on the weights is used.

ResNet-18 \cite{he2016deep} is a ResNet model containing 18 convolutional layers  for CIFAR-100 classification~\cite{he2016deep}. Input images are down-scaled to $1/8$ of their original sizes after the 18 convolutional layers, and then fed into a fully-connected layer for the 100-class classification. The output channel numbers of 1-3 conv layers, 4-8 conv layers, 9-13 conv layers, and 14-18 conv layers are $64$, $128$, $256$, and $512$, respectively.  The batch size is $64$. The training stage lasts for $100$ epochs in total. No $\ell_2$ regularization on the weights is used.

\subsection{Additional Experiments of ResNet-18 on CIFAR-100}

We further illustrate Generic Adam with different $r=\{0, 0.25, 0.5, 0.75, 1\}$, RMSProp, and AMSGrad with an alternative base learning rate $\alpha = 0.01$ on ResNet-18.  We do cut-off by taking $\alpha_t = 0.001$ if $t < 2500$. Note that $\alpha_t$ is still non-increasing. The motivation is that at the very beginning the learning rate $\alpha_{t} = \frac{0.01}{\sqrt{t}}$ could be large which would deteriorate the performance. The performance profiles are also exactly in accordance with the analysis in theory, \textit{i.e.}, larger $r$ leads to a faster training process.

\begin{figure*}[htpb]
\centering
\subfigure[]{\includegraphics[width=0.32\linewidth]{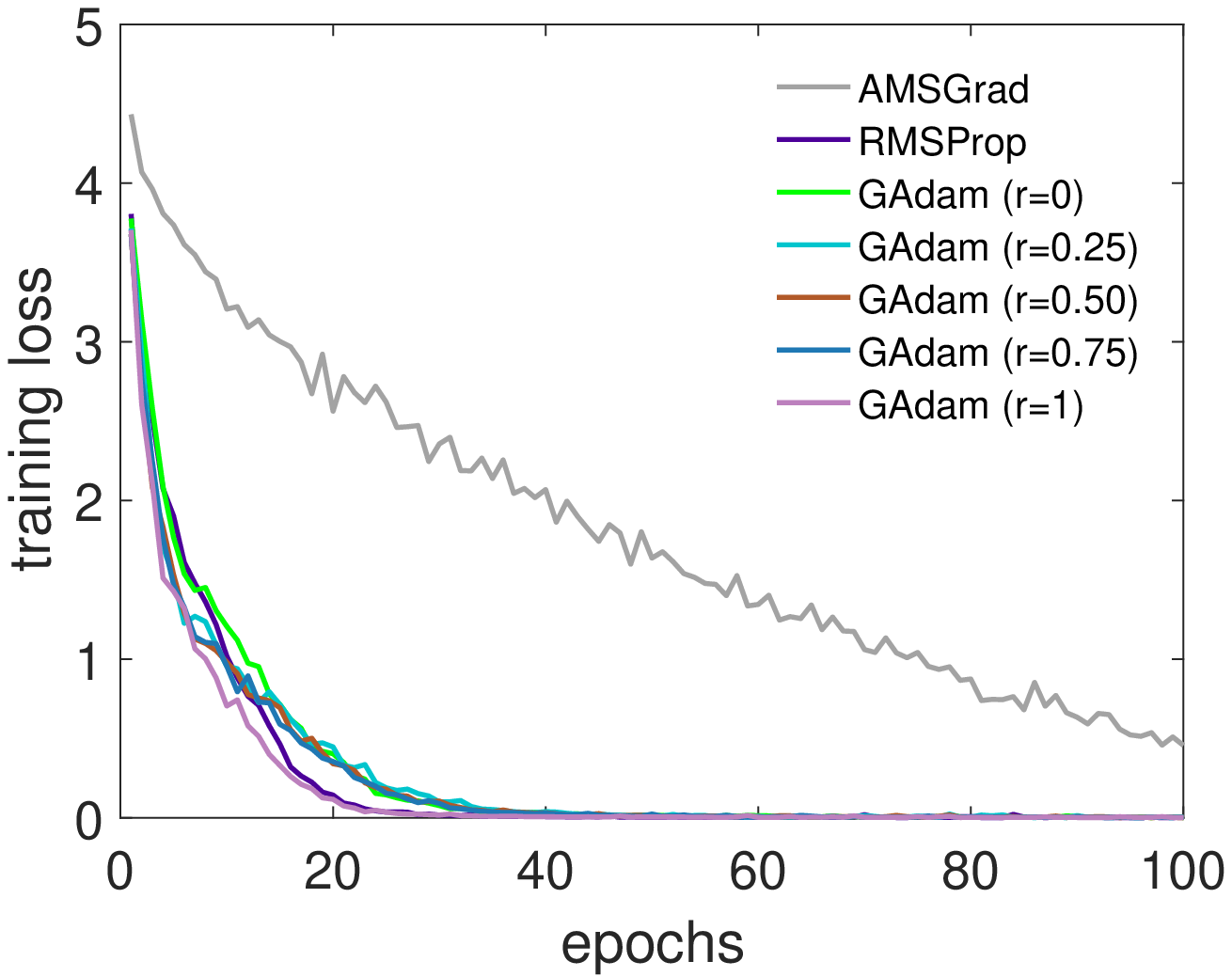}}\label{fig:resnet_a_01}
\subfigure[]{\includegraphics[width=0.32\linewidth]{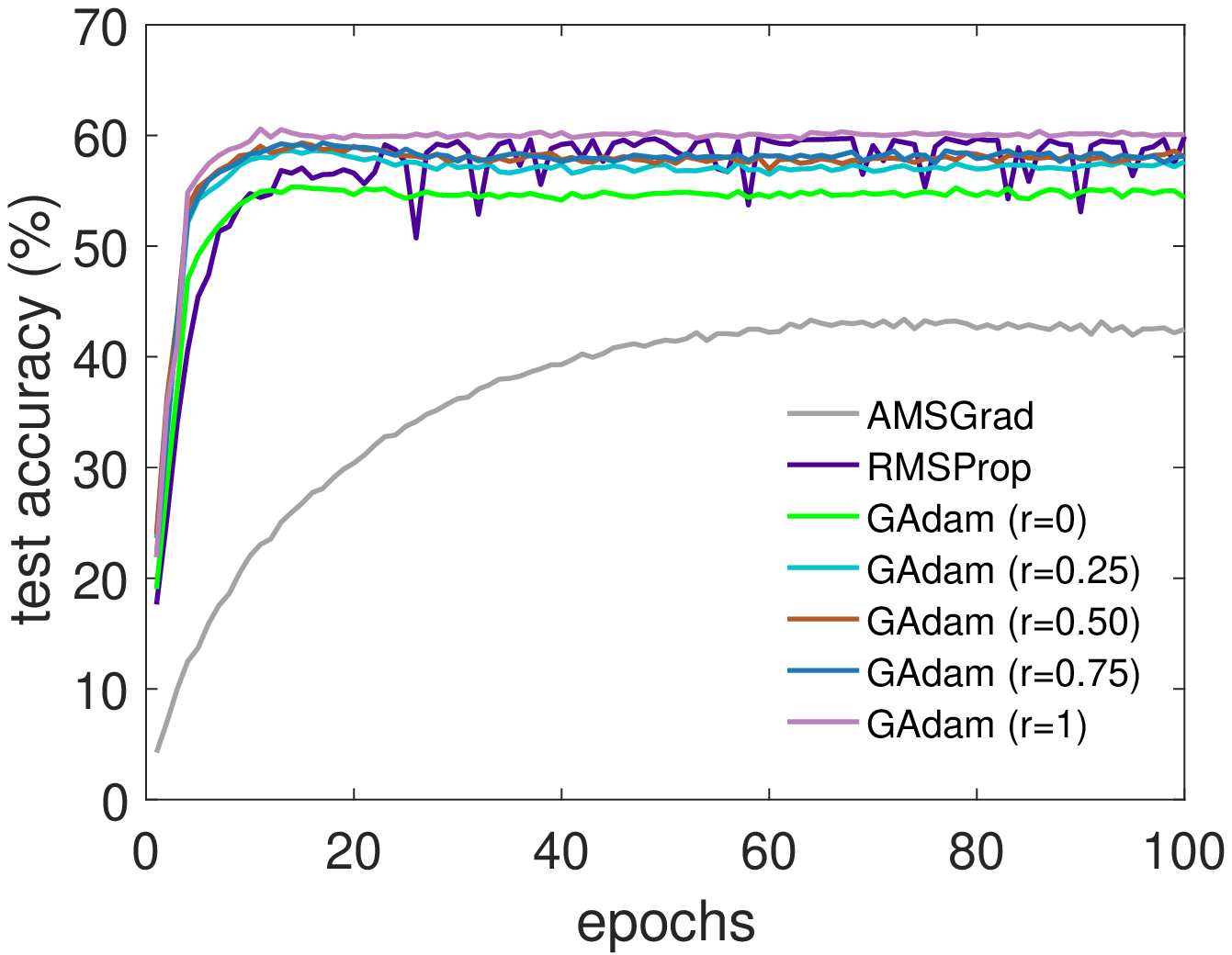}}\label{fig:resnet_b_01}
\subfigure[]{\includegraphics[width=0.32\linewidth]{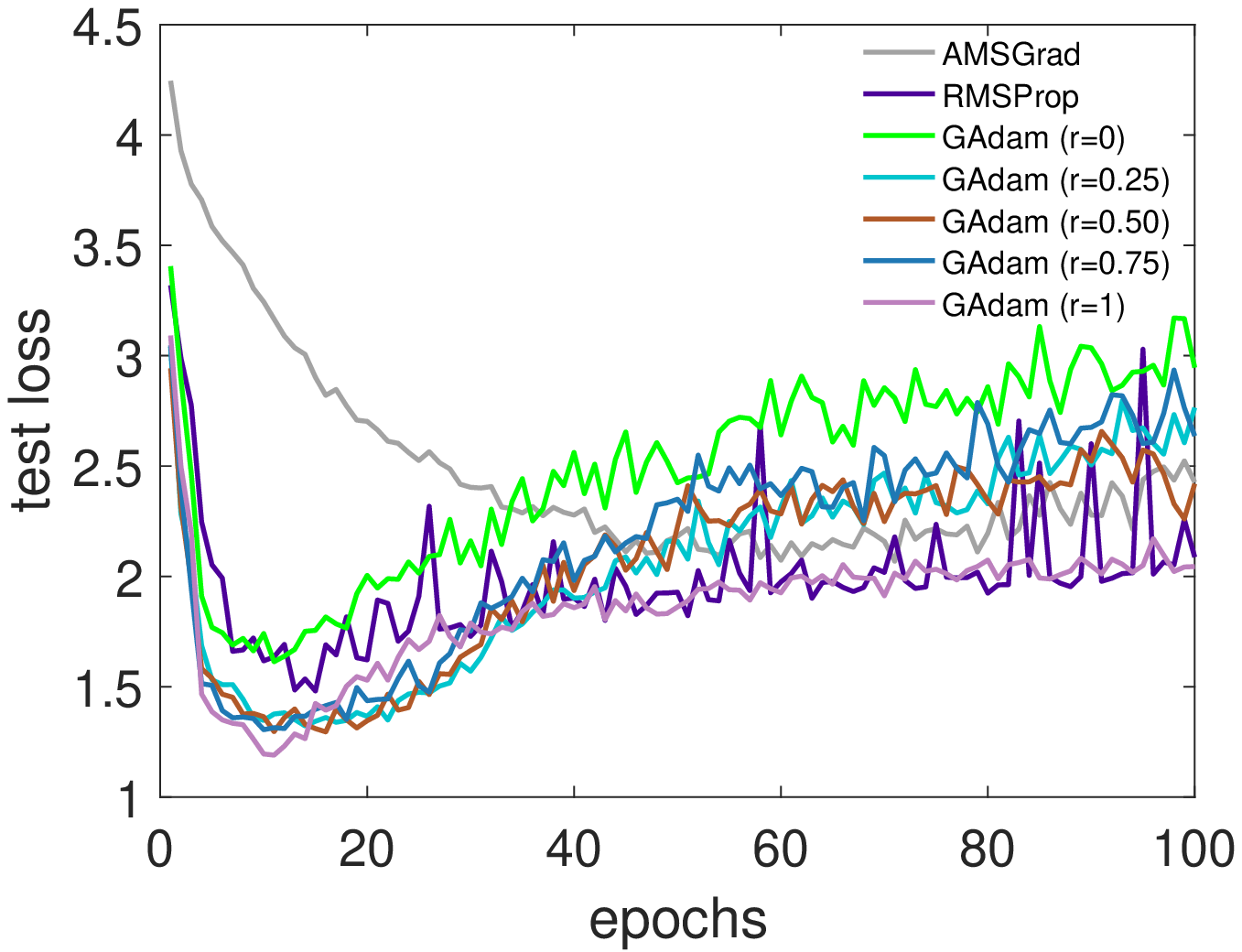}}\label{fig:resnet_c_01}
\caption{Performance profiles of Generic Adam with $r=\{0, 0.25, 0.5, 0.75, 1\}$, RMSProp, and AMSGrad for training ResNet on the CIFAR-100 dataset. Figures (a), (b), and (c) illustrate training loss vs. epochs, test accuracy vs. epochs, and test loss vs. epochs, respectively.}
\label{fig:ResNet-01}
\end{figure*}
\end{document}